\documentclass[final,onefignum,onetabnum]{siamart190516}

\usepackage{amsmath}
\usepackage{booktabs}       % professional-quality tables
\allowdisplaybreaks
\usepackage{amssymb}
\usepackage{enumitem}
\usepackage{bm}
\usepackage{algorithm,algorithmic}

\usepackage{autonum}
\usepackage[caption=false,font=footnotesize,subrefformat=parens,labelformat=parens]{subfig}

\DeclareMathOperator*{\minimize}{minimize}
\DeclareMathOperator*{\maximize}{maximize}
\DeclareMathOperator*{\st}{subject\ to}

\def\diag{\text{diag}}

\def\ln{\text{ln}}

\def\bea{\begin{equation}\begin{aligned} }
\def\ena{\end{aligned}\end{equation} }
\def\bee{\begin{equation}}
\def\ene{\end{equation}}

\renewcommand{\vec}[1]{\mathbf{#1}}
\newcommand{\note}[1]{\ignorespaces}

\newtheorem{assum}{Assumption}

\def\T{\mathsf{T}}

\def\bone{{\mathbf{1}}}

\def\bxi{\boldsymbol{\xi}}

\newcommand{\mathletter}[1]{%
	\expandafter\newcommand\csname b#1\endcsname{\mathbb #1}
	\expandafter\newcommand\csname c#1\endcsname{\mathcal #1}
	\expandafter\newcommand\csname f#1\endcsname{\mathfrak #1}
	\expandafter\newcommand\csname til#1\endcsname{\widetilde #1}
	\expandafter\newcommand\csname ha#1\endcsname{\widehat #1}
	\expandafter\newcommand\csname bf#1\endcsname{\bf #1}
	\expandafter\newcommand\csname s#1\endcsname{\mathsf #1}
}%
\def\mathletters#1{\mathlettersB #1,,}
\def\mathlettersB#1,{\ifx,#1,\else\mathletter #1\expandafter\mathlettersB\fi}
\mathletters{A,B,C,D,E,F,G,H,I,J,K,L,M,N,O,P,Q,R,S,T,U,V,W,X,Y,Z}

\newcommand{\mathletterl}[1]{%
	\expandafter\providecommand\csname v#1\endcsname{\vec{#1}}
}%
\def\mathlettersl#1{\mathlettersC #1,,}
\def\mathlettersC#1,{\ifx,#1,\else\mathletterl #1\expandafter\mathlettersC\fi}
\mathlettersl{a,b,c,d,e,f,g,h,i,j,k,l,m,n,o,p,q,r,s,t,u,v,w,x,y,z}

\newsiamremark{remark}{Remark}
\newsiamremark{assumption}{Assumption}
\crefname{assumption}{Assumption}{Assumption}
\newsiamremark{example}{Example}
\headers{DSGT for Non-convex ERM}{Jiaqi Zhang and Keyou You}

\title{Decentralized Stochastic Gradient Tracking for Non-convex Empirical Risk Minimization\thanks{Submitted to the editors DATE.
		\funding{This work was supported by the National Natural Science Foundation of China under Grant 61722308.}}}

\author{Jiaqi~Zhang\thanks{Department of Automation, and BNRist, Tsinghua University, Beijing 100084, China. 
		(\email{zjq16@mails.tsinghua.edu.cn}, \email{youky@tsinghua.edu.cn}).}
	\and Keyou~You\footnotemark[2]}

\ifpdf
\hypersetup{
  pdftitle={Decentralized Stochastic Gradient Tracking for Non-convex Empirical Risk Minimization},
  pdfauthor={Jiaqi Zhang and Keyou You}
}
\fi

\begin{document}

\maketitle

\begin{abstract}
	This paper studies a decentralized stochastic gradient tracking (DSGT) algorithm  for non-convex empirical risk minimization problems over a peer-to-peer network of nodes, which is in sharp contrast to the existing DSGT only for convex problems. To ensure exact convergence and handle the variance among decentralized datasets, each node performs a stochastic gradient (SG) tracking step by using a mini-batch of samples, where the batch size is designed to be proportional to the size of the local dataset. We explicitly evaluate the convergence rate of DSGT  with respect to the number of iterations  in terms of algebraic connectivity of the network, mini-batch size, gradient variance, etc. Under certain conditions, we further show that DSGT has a network independence property in the sense that the network topology only affects the convergence rate up to a constant factor. Hence, the convergence rate of DSGT can be comparable to the centralized SGD method. Moreover, a linear speedup of DSGT with respect to the number of nodes is achievable for some scenarios. Numerical experiments for neural networks and logistic regression problems on CIFAR-10 finally illustrate the advantages of DSGT.
\end{abstract}



\section{Introduction}

\subsection{Decentralized empirical risk minimization}
Empirical risk minimization (ERM) problems arise from many machine learning applications, which train models by minimizing some risk function related to a set of samples \cite{bottou2018optimization}. In many problems, the dataset is very large and it may be too slow to train on a single computing node. Moreover, the dataset can be collected from spatially distributed nodes and each local dataset is only be accessed by a single node. Both cases suggest the necessity of using multiple nodes for decentralized training.

In this work, we consider the problem where an ERM problem is solved using $n$ computing nodes in a decentralized manner. Each node $i$ collects or is assigned a local dataset $\cD_i=\{d_{1}^{(i)},\ldots,d_{N_i}^{(i)}\}$ with $N_i$ samples, and a local objective function $f_i(\vx)$ is defined associated with $\cD_i$. An ERM problem aims to solve an optimization problem of the form
\bee\label{obj2}
\minimize_{\vx\in\bR^m}\ f(\vx)=\sum_{i=1}^{n}f_i(\vx),\ f_i(\vx)\triangleq\sum_{u=1}^{N_i}l(\vx;d_{u}^{(i)})
\ene
where $l(\vx;d_{u})$ is the loss function of sample $d_{u}$ given parameter $\vx$, and is assumed to be continuously differentiable in this work. 
We allow local datasets to have different sizes and be sampled from different distributions.

Two types of networks  are commonly adopted for communications among nodes. (a) master-slave networks \cite{dean2012large,mcmahan2017communication}, where a  master node in the network collects and aggregates information (e.g. local gradients) from all slave nodes per iteration for updating. 
(b) peer-to-peer networks, where each node or worker performs local updates and communicates only with neighboring nodes to share information \cite{nedic2017network,lian2017can,assran2018stochastic,pu2018distributed2,zhang2020distributed}. The associated algorithms are referred to as {\em centralized} and {\em decentralized} (or distributed) algorithms  in this work, respectively.  Since the centralized one is vulnerable to the damage of master node, has poor privacy protection, and suffers from the communication bottleneck of the master node \cite{lian2017can,assran2018stochastic}, this paper focuses on the design of \emph{decentralized} algorithms.

The interactions between nodes is modeled by a graph $\cG=(\cV,\cE)$, where $\cV=\{1,\ldots,n\}$ is the set of nodes, $\cE$ is the set of edges and $(i,j)\in\cE$ if and only if nodes $i$ and $j$ can communicate with each other. The set $\cN_i=\{j|(i,j)\in\cE\}$ is called the neighbors of node $i$. In decentralized optimizations,  each node only performs local computations and communicates with its neighbors.

\subsection{Our contributions}

We study a decentralized stochastic gradient tracking (DSGT) algorithm to solve \cref{obj2} over a peer-to-peer network. DSGT allows nodes to only have access to local datasets. Each node computes stochastic gradients (SGs) with a mini-batch of samples from its local dataset, and communicates with its neighbors at each iteration for cooperative optimization.

If all local objective functions $f_i,i\in\cV$ are Lipschitz smooth and non-convex, we provide an explicit  non-asymptotic convergence rate of DSGT in terms of key parameters of the problem. Specifically, all nodes in DSGT converge to the same stationary point at a rate of $O\big(\frac{1}{\sum_{t=1}^K\gamma_t}\big(D+\sigma_s^2\sum_{t=1}^K\gamma_t^2+\frac{\rho^2\sigma_s^2}{(1-\rho)^3}\sum_{t=1}^K\gamma_t^3\big)\big)$, where $K$ is the number of iterations, $\gamma_k$ is the stepsize  in the $k$-th iteration, $D$ is related to the initial values, $\sigma_s^2$ is the variance of local SGs, and $(1-\rho)\in(0,1]$ is the algebraic connectivity \cite{chung1997spectral,aragues2014distributed} of the communication graph $\cG$ (formally defined in \cref{assum1}). In particular, the convergence rate becomes $O(\frac{D\sigma_s}{\sqrt{K}}+\frac{\rho^2 D^2}{(1-\rho)^3K})$ if an appropriate constant stepsize is adopted, and achieves a rate of $O(1/k^{1-p})$ (or $O(\ln(k)/\sqrt{k})$) for diminishing stepsizes $\gamma_k=O(1/k^p),p\in(0.5,1)$ (or $p=0.5$). We also show that the function value converges to the optimal value if the objective function is further convex.

Moreover, our theoretical result reveals that the convergence rate of DSGT can be independent of the network topology under reasonable conditions. Specifically, if the stepsize satisfies $\frac{\rho^2}{(1-\rho)^3}\sum_{t=1}^k\gamma_t^3=O(\sum_{t=1}^k\gamma_t^2)$, then the convergence rate of DSGT with respect to (w.r.t) the number of iterations is comparable to the centralized SGD, and hence DSGT behaves as if the network does not exist. For constant stepsizes, this condition transforms into running the algorithm for sufficiently many iterations. For diminishing stepsize it is satisfied if the initial stepsize is small enough.

Speedup is an important property to evaluate the scalability of a decentralized algorithm. We compare DSGT with $n$ nodes to the centralized SGD running on a single node, and then study the speedup of DSGT under different scenarios. We show that the convergence rate of DSGT can match that of the centralized SGD w.r.t. number of iterations under some assumptions, and hence a linear speedup in running time w.r.t. number of nodes is achievable. Nevertheless, we also suggest that there are cases, e.g., local cost functions have very different Lipschitz constants, where the speedup may be sublinear for both DSGT and existing algorithms, which appears to be overlooked in the literature (e.g. \cite{lian2017can,tang2018d,assran2018stochastic}).

Finally, we conduct several experiments on neural networks and logistic regression problems, and show the advantage of DSGT over two recent major decentralized algorithms.

\subsection{Related work}
DSGT was first introduced in \cite{pu2018distributed,pu2018distributed2} for stochastic optimization problems with strongly convex functions. Since they do not study the ERM problem, it does not exploit the inherent feature of  sampling from local datasets, which leads to some interesting results in this work, e.g., we show that DSGT can be even faster than the centralized SGD for some ERM problems (c.f. \cref{exm2}), and achieve the speedup property. Moreover, they show the linear convergence of DSGT to a neighborhood of an optimal solution via the spectral radius of some matrix $M\in\bR^{3\times 3}$. Clearly, such an approach is no longer applicable here as we cannot expect linear convergence for general convex or non-convex functions. In contrast, under weaker assumptions on the variance of SGs, this work deals with both constant and diminishing stepsizes for non-convex functions and derives the non-asymptotic convergence rate with an explicit dependence on key parameters of ERM problems.

A recent work \cite{lu2019gnsd} studies a combine-then-adapt \cite{yuan2019performance,yuan2018variance} variant of DSGT with constant stepsize for non-convex functions. However, their assumptions and convergence rate results are more conservative, and there is no analysis of dependence on the algebraic connectivity or speedup properties. Interestingly, a very recent work \cite{xin2020improved} on stochastic optimization has been posted on the arXiv several days before we finish the revision. It also derives conditions for the network independence property of DSGT for non-convex functions, and obtains the convergence rate of DSGT if the Polyak-{\L}ojasiewicz condition is satisfied.   Moreover, it only characterizes the convergence rate for constant stepsizes, and does not compare with the centralized SGD or analyze the speedup in the context of ERM. Finally, the deterministic version of DSGT and its variants are studied in \cite{nedic2017achieving,qu2017harnessing,xu2015augmented}, where a linear convergence rate can be achieved for strongly convex and Lipschitz smooth objective functions.

A striking advantage of DSGT lies in the independence of its convergence rate on the differences between local objective functions. This feature is important in applications where the distributions between local datasets are large (e.g. federated learning \cite{mcmahan2017communication}). In this sense, DSGT improves most existing decentralized stochastic algorithms \cite{lian2017can,assran2018stochastic,hao2019linear}. For example, the D-PSGD\cite{lian2017can} converges at a rate of $O\left(\frac{\sigma}{\sqrt{nK}}+\frac{n^{{1}/{3}} \zeta^{{2}/{3}}}{K^{{2}/{3}}}\right)$, where $\zeta$ reflects the discrepancy among local cost functions, i.e, $\frac{1}{n}\sum_{i=1}^n\|\nabla f_i(\vx)-\nabla f(x)\|^2\leq\zeta^2,\forall \vx$. This assumption is strong and D-PSGD can even diverge for non-convex functions if it is not satisfied \cite[Example 3]{chang2020distributed}. To remove it,  D$^2$ \cite{tang2018d} has been proposed with fixed stepsizes.  In comparison, DSGT allows a wider class of weight matrices that often lead to a smaller $\rho$ and hence a faster convergence rate than D$^2$. We shall compare DSGT with D$^2$ and D-PSGD in detail later.

There are many other decentralized algorithms that focus on different aspects. (a) Algorithms using \emph{full} local gradients such as DGD\cite{nedic2009distributed}, DDA\cite{agarwal2010distributed}, EXTRA\cite{shi2015extra},  MSDA\cite{scaman2017optimal} and MSPD\cite{scaman2018optimal}. MSDA is shown to have optimal dependence on the algebraic connectivity of the network, but it requires computing the gradient of a Fenchel conjugate function per iteration, which makes it difficult to adapt to SGs and complex cost functions, e.g., neural networks. (b) Algorithms for general directed networks such as SGP \cite{assran2018stochastic,nedic2016stochastic}, \cite{xie2018distributed}, and \cite{xin2019distributed}. In particular, SGP reduces to D-PSGD for symmetric weight matrix in undirected graphs. (c) Algorithms with asynchronous updates such as AD-PSGD\cite{lian2018asynchronous}, Asy-SONATA\cite{tian2018achieving}, AsySPA \cite{zhang2019asyspa,assran2018asynchronous}, APPG \cite{zhang2019asynchronous}, and \cite{hendrikx2019asynchronous}. (d) Algorithms with compressed or efficient communication \cite{lan2017communication,shen2018towards,koloskova2019decentralized,tang2018communication,lu2020moniqua}, unreliable network \cite{tang2018distributed}, stale gradient information \cite{sirb2018decentralized,assran2018stochastic},  stochastic approximation \cite{bianchi2013performance}, information diffusion \cite{yuan2019performance}, distributed SGD with momentum \cite{hao2019linear}, etc. It is interesting to extend DSGT to the above situations in future works.

Finally, it is also possible to accelerate DSGT with momentum \cite{hao2019linear,qu2019accelerated} or variance reduction methods \cite{johnson2013accelerating,roux2012stochastic,yuan2018variance}. For example, \cite{mokhtari2016dsa} combines EXTRA and SAGA \cite{defazio2014saga} to obtain DSA, which converges linearly. However, its performance for non-convex functions and speedup w.r.t. number of workers are unclear.

{\bf Notations}~~~Throughout this paper, $\|\cdot\|$ denotes the $l_2$ norm of vectors or spectral norm of matrices. $\|\cdot\|_\sF$ denotes the Frobenius norm. $\nabla f$ denotes the gradient of $f$ and $\partial f$ denotes a SG (formally defined in Section \ref{sec2}).  $I$ denotes the identity matrix, and $\bone$ denotes the vector with all ones, the dimension of which depends on the context. A matrix $W$ is called doubly stochastic if $W\bone=\bone$ and $W^\T\bone=\bone$. 

\section{The Decentralized Stochastic Gradient Tracking algorithm (DSGT)}\label{sec2}

The DSGT is given in Algorithm \ref{alg_DSGT}. At each iteration $k$, each node $i$ receives $\vx_{j,k}$ and $\vy_{j,k}$ from each neighbor $j$, and uses them to update $\vx_{i,k+1}$ by \cref{eq1_alg}, which is an inexact SG update followed by a weighted average. Then, node $i$ computes the SG based on $\vx_{i,k+1}$ and a mini-batch samples $\xi_{i,k+1}$, and use it to update $\vy_{i,k+1}$ by \cref{eq2_alg}, which is the so-called SG tracking step and produces an inexact SG of the \emph{global} objective function. Finally,  $\vx_{i,k+1}$ and $\vy_{i,k+1}$ are broadcast to neighbors and all nodes enter the next iteration. Thus, the local SG and dataset are not shared with other nodes. $[W]_{ij}$ used in the weighted average in \cref{eq1_alg} and \cref{eq2_alg} is the $(i,j)$-th element of $W$, and $W\in\bR^{n\times n}$ is a doubly stochastic matrix with $[W]_{ij}=0$ if $i$ and $j$ cannot communicate with each other, i.e., $(i,j)\notin\cE$.

From a global viewpoint, Algorithm \ref{alg_DSGT} can be written in the following compact form
\bea\label{alg}
X_{k+1}&=W(X_k-\gamma_k Y_k)\\
Y_{k+1}&=WY_k+\partial_{k+1}-\partial_k
\ena
where $X_k$, $Y_k$ and $\partial_k$ are $n\times m$ matrices defined as
\bea
&X_k=[\vx_{1,k},\vx_{2,k},\ldots,\vx_{n,k}]^\T,\ Y_k=[\vy_{1,k},\vy_{2,k},\ldots,\vy_{n,k}]^\T,\\
&\partial_k=[\partial f_1(\vx_{1,k};\xi_{1,k}),\ldots,\partial f_n(\vx_{n,k};\xi_{n,k})]^\T
\ena
with $\partial f_i(\vx_{i,k};\xi_{i,k})=\sum\nolimits_{d\in\xi_{i,k}}\nabla_x l(x;d)$,
and $\xi_{i,k}$ is a set of $\eta N_i (\eta\in(0,1))$ data points that are uniformly and randomly sampled from the local dataset $\cD_i$ at iteration $k$. The initial states of \cref{alg} is $X_1=WX_0$ and $Y_1 =\partial_1$.

\subsection{Mini-batch size proportional to the size of local dataset} 

A notable feature of DSGT is that the mini-batch size is proportional to local dataset size, i.e., $\eta$ in Algorithm \ref{alg_DSGT} is the same among nodes, while the existing decentralized algorithms use the same batch size for each node regardless of the local dataset  size, see e.g. \cite{lian2017can,tang2018d,assran2018stochastic}. We now provide a motivation for it. 

Let $\cF_k=\{X_0,\bxi_0,Y_0,\ldots,X_{k-1},\bxi_{k-1},Y_{k-1},X_k,\bxi_k, Y_k, X_{k+1}\}$ be the history sequence of random variables and define
\bea\label{eq3_sec2}
%\noeqref{eq3_sec2}
\bar\vx_k&\triangleq\frac{1}{n}\sum_{i=1}^{n}\vx_{i,k}=\frac{1}{n}X_k^\T\bone,\\ 
\bar\vy_k&\triangleq\frac{1}{n}\sum_{i=1}^{n}\vy_{i,k}=\frac{Y_k^\T\bone}{n}\overset{\cref{alg}}{=}\frac{1}{n}\left(Y_{k-1}+\partial_k-\partial_{k-1}\right)^\T\bone=\frac{1}{n}\left(Y_1+\partial_k-\partial_1\right)^\T\bone=\frac{\partial_k^\T\bone}{n},\\
\vg_k&\triangleq\frac{1}{n}\sum_{i=1}^{n}\nabla f_i(\vx_{i,k})=\frac{\nabla F(X_k)^\T\bone}{n}, \text{ where } \nabla F(X_k)\triangleq[\nabla f_1(\vx_{1,k}),\ldots,\nabla f_n(\vx_{n,k})]^\T\\
\ena
and we used $(Y_k-\partial_{k})^\T\bone=(Y_{k-1}-\partial_{k-1})^\T\bone$. Then,  it holds that $
\bE[\bar\vy_k|\cF_{k-1}]=\frac{1}{n}\sum_{i=1}^{n}\bE[\partial f_i(\vx_{i,k};\xi_{i,k})]=\frac{1}{n}\sum_{i=1}^{n}\frac{\eta N_i}{N_i}\sum_{u=1}^{N_i}\nabla_x l(\vx_{i,k};d_u^{(i)})=\eta \vg_{k}.
$
If all $\vx_{i,k}$ almost achieve consensus, i.e., $\vx_{i,k}\approx\bar\vx_{k},\forall i$,  it follows from \cref{alg} that the update of $\bE[\bar\vx_{k}]$ is just a standard gradient descent step, i.e.,
\bee
\bE[\bar\vx_{k+1}|\cF_{k-1}]=\bar\vx_{k}-\gamma_k\eta \vg_k\approx\bar\vx_{k}-\frac{\gamma_k\eta}{n}{\sum\nolimits_{i=1}^{n}\nabla f_i(\bar{\vx}_k)} =\bar\vx_{k}-\frac{\gamma_k\eta}{n}\nabla f(\bar\vx_{k})
\ene
which promises the convergence of $\bar\vx_{k}$ to an optimal solution of \cref{obj2} if appropriate stepsizes are designed. In fact, the consensus is guaranteed due to the weighted averaging in \cref{eq1_alg} and \cref{eq2_alg}.

However, if all nodes use the same mini-batch size, say $M$, then
\bee
\bE[\bar\vy_k|\cF_{k-1}]=\frac{1}{n}\sum\nolimits_{i=1}^{n}\frac{M}{N_i}\sum\nolimits_{u=1}^{N_i}\nabla_x l(\vx_{i,k};d_u^{(i)})\approx\frac{M}{n}\sum\nolimits_{i=1}^{n}\frac{1}{N_i}\nabla f_i(\bar\vx_{k}).
\ene
The resulting algorithm is actually minimizing a weighted sum of local objective functions with weights inversely proportional to the size of local dataset, which is generally different  from that of \cref{obj2}. This problem also exists in the existing works \cite{lian2017can,tang2018d,assran2018stochastic}.

\begin{algorithm}[!t]
	\caption{The Decentralized Stochastic Gradient Tracking Method (DSGT) --- from the view of node $i$}\label{alg_DSGT}
	\begin{minipage}{\textwidth}
		\begin{algorithmic}[1]
			\REQUIRE Initial states $\vx_{i,0}, \vy_{i,0}=0, \vs_{i,0}=0$, stepsizes $\{\gamma_k\}$, weight matrix $W$, batch size proportion $\eta\in(0,1]$, and the maximum number of iterations $K$.
			\FOR {$k=0,1,\ldots,K-1$}
			\STATE Receive or fetch $\vx_{j,k}$ and ${\vy}_{j,k}$ from neighbors, and update $\vx_{i,k+1}$ by the weighted average:
			\bee\label{eq1_alg}
			\vx_{i,k+1}=\sum\nolimits_{j\in\cN_i}[W]_{ij}({\vx}_{j,k}-\gamma_k\vy_{j,k}).
			\ene
			\STATE Uniformly randomly sample a mini-batch $\xi_{i,k+1}$ of size $\eta N_i$\footnote{One can also perform Bernoulli sampling on $\cD_i$ with the probability of success $\eta$, which essentially leads to the
			same convergence result.} from the local dataset $\cD_i$, compute the stochastic gradient and set $\vs_{i,k+1}=\partial f_i(\vx_{i,k+1};\xi_{i,k+1})=\sum_{d\in\xi_{i,k+1}}\nabla_x l(\vx_{i,k+1};d)$.
			\STATE Update $\vy_{i,k+1}$ and  ${\vx}_{i,k+1}$ by
			\bea\label{eq2_alg}
			\vy_{i,k+1}&=\sum\nolimits_{j\in\cN_i}[W]_{ij}{\vy}_{j,k}+\vs_{i,k+1}-\vs_{i,k}
			\ena
			\STATE Send ${\vx}_{i,k+1}$ and ${\vy}_{i,k+1}$ to all neighbors.
			\ENDFOR
%			\ENSURE $\frac{1}{n}\sum_{i=1}^n\vx_{i,K}$.
		\end{algorithmic}
	\end{minipage}
\end{algorithm}

\subsection{Comparison with D$^2$}\label{sec2.2}

DSGT shares some similarities with D$^2$\cite{tang2018d} except the batch size and time-varying stepsizes. In a compact form, D$^2$ has the following update rule
\bee\label{eq1_sec2}
X_{k+1}=2WX_k-WX_{k-1}-\gamma W(\partial F(X_{k};\bxi_{k})-\partial F(X_{k-1};\bxi_{k-1}))
\ene
and DSGT (with constant stepsize) can be rewritten as the following form by eliminating $Y_{k}$
\bee\label{eq4_sec2}
X_{k+1}=2WX_k-W^2X_{k-1}-\gamma W(\partial F(X_{k};\bxi_{k})-\partial F(X_{k-1};\bxi_{k-1})).
\ene
It turns out that the only difference is the coefficient matrix before $X_{k-1}$, which however requires a completely different analysis and provides the following advantages:
\begin{itemize}[leftmargin=14pt]
	\item The convergence of D$^2$ requires $W+\frac{1}{3}I$ to be positive definite while DSGT does not. This implies that DSGT can adapt to a wider class of communication graphs. In fact, the optimal weight matrix $W$ with the smallest $\rho$ (see \cref{assum1})  often violates that $W\succ -\frac{1}{3}I$ \cite{xiao2004fast} and leads to the divergence of D$^2$, which is illustrated in the following example.

\begin{figure}[!ht]
	\centering
	\includegraphics[width=0.3\linewidth]{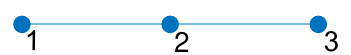}
	\caption{A graph.}
	\label{fig_graph}
\end{figure}
\begin{example}
Consider the communication topology in Fig. \ref{fig_graph}. From \cite{xiao2004fast}, the optimal weight matrix minimizing $\rho$ is given by
$$
W=\begin{bmatrix}
	0.5 & 0.5 & 0   \\
	0.5 & 0   & 0.5 \\
	0   & 0.5 & 0.5
\end{bmatrix},
$$
which has a negative eigenvalue -0.5. Thus, it does not satisfy $W\succ -\frac{1}{3}I$. 

We now show that D$^2$ (or EXTRA in the deterministic case) indeed diverges for some problems with such a weight matrix $W$.  Consider a simple least square problem where $f_i(x)=\frac{1}{2}(x-a_i)^2$ and the global cost function is $f(x)=\sum_{i=1}^{3}f_i(x)$. It follows from \cref{eq1_sec2} that D$^2$ has the update rule
\bee\label{eq1_adx7}
X_{k+1} = (2-\gamma)WX_k - (1-\gamma)WX_{k-1}
\ene
where we use the full gradient $\nabla f_i(x_i)=x_i-a_i$ instead of a SG. To study the behavior of \cref{eq1_adx7}, we rewrite it as
\bee
\begin{bmatrix}
	X_{k+1}-\bone(x^\star)^\T \\
	X_{k}-\bone(x^\star)^\T
\end{bmatrix}=
\underbrace{\begin{bmatrix}
		(2-\gamma)W & -(1-\gamma)W \\
		I           & 0
\end{bmatrix}}_{H}
\begin{bmatrix}
	X_{k}-\bone(x^\star)^\T\\
	X_{k-1}-\bone(x^\star)^\T
\end{bmatrix}
\ene
where $x^\star$ is the minimum point of $f(x)$. Thus, $X_k$ converges to $\bone(x^\star)^\T$ only if the spectral radius of $H$ is less than 1, which however does not hold for any $\gamma < 0.5$. Hence D$^2$ \emph{diverges} for \emph{any} $\gamma < 0.5$. Note that DSGT converges for any $\gamma \leq 1$.
\end{example}
	\item $\vy_{i,k}$ in DSGT is to track the aggregated local SGs of nodes and the expectation of $\bar\vy_{k}$ is an approximation of the full gradient. This fact offers an insight into the algorithm and brings convenience to extend DSGT. For example, one may consider treating $\vy_{i,k}$ as true global SGs and applying existing accelerated methods such as momentum to $\vy_{i,k}$. 
\end{itemize}

\subsection{Comparison with D-PSGD}\label{sec2.3}

DSGT also improves D-PSGD \cite{lian2017can}  (or SGP \cite{assran2018stochastic} over static undirected graphs) by removing the assumption $\frac{1}{n}\sum_{i=1}^n\|\nabla f_i(\vx)-\nabla f(x)\|^2\leq\zeta^2,\forall \vx$, without which D-PSGD may diverge \cite[Example 3]{chang2020distributed}. With some mathematical manipulations,
we can rewrite \cref{alg} with constant stepsize as
\bee\label{eq2_sec2}
X_{k+1}=\underbrace{WX_{k}-\gamma\partial F(X_{k};\bxi_{k})}_{\text{D-PSGD}}-\underbrace{\gamma\big(\textstyle\sum_{t=0}^{k-1}W(W-I)^t\partial F(X_{k-t};\bxi_{k-t})-\partial F(X_{k};\bxi_{k})\big)}_{\text{momentum}}
\ene
In this view, DSGT can be regarded as D-PSGD with a carefully designed momentum. Note that the momentum in \cref{eq2_sec2} uses history information from neighbors, which is quite different from naive acceleration schemes by directly applying Nesterov's momentum to \emph{local} SGs.

We finally note that each node in DSGT transmits two $m$-dimensional variables per iteration, which can be a drawback in high-latency and low-bandwidth networks. Nevertheless,  empirical results in Section \ref{sec4} show that DSGT needs less training time to achieve the same degree of accuracy. Future works can focus on the design of communication-efficient versions of DSGT \cite{lan2017communication,shen2018towards,koloskova2019decentralized,lu2020moniqua}. 

\section{Theoretical result}\label{sec3}

This section establishes the non-asymptotic convergence rate of DSGT. To this end, we make the following assumption.
\begin{assum}\label{assum1}
	The following assumptions are made throughout this paper.
	\begin{enumerate}[leftmargin=15pt, labelsep=4pt,label=(\alph*)]
		\item All local objective functions are Lipschitz smooth, i.e., there exist $L_i>0,i\in\cV$ such that for any $\vx,\vy\in\bR^m$,
		$
		\|\nabla f_i(\vx)-\nabla f_i(\vy)\|\leq L_i\|\vx-\vy\|.
		$
		Moreover, the global objective function $f$ is bounded from below.

		\item The communication graph is connected\footnote{A graph is connected means that there exists a path between any two nodes in the graph. This is obviously a necessary condition.} and the weight matrix $W$ is doubly stochastic. 
		\item The variances of local SGs satisfy that
		\bee
		\bE[\|\partial f_i(\vx;\xi_{i,k})-\eta\nabla f_i(\vx)\|^2|\cF_{k-1}]\leq\sigma_i^2+\lambda^2\|\eta\nabla f(\vx)\|^2,\ \forall i,\vx,k.
		\ene
		for some $\sigma_i,\lambda\geq0,i\in\cV$.
	\end{enumerate}
\end{assum}

Define $\rho\triangleq\|W-\frac{1}{n}\bone\bone^\T\|$ which is strictly less than $1$ under Assumption \ref{assum1}(b), $\sigma_s^2\triangleq\sum_{i=1}^n\bE[\|\partial f_i(\vx;\xi_{i,k})-\eta\nabla f_i(\vx)|\cF_{k-1}\|^2]\leq\sum_{i=1}^{n}\sigma_i^2$, $f^\star\triangleq\inf_{\vx}f(\vx)$, and $L \triangleq\max_{i\in\cV} L_i$. We explore \cref{assum1} in more details.
\begin{itemize}[leftmargin=14pt]
	\item \cref{assum1}(a) is standard in non-convex optimization \cite{lan2020first}. In the context of ERM, it is worth noting that $L_i$ generally depends on the local dataset size $N_i$ since $f_i$ is the summation of $N_i$ functions. 
	\item \cref{assum1}(b) arises from the decentralized optimization \cite{lian2017can,assran2018stochastic} and $1-\rho$ is the well-studied \emph{algebraic connectivity} of the network \cite{chung1997spectral,aragues2014distributed}\footnote{The original definition of algebraic connectivity is the second smallest eigenvalue of the Laplacian matrix associated with the graph. Here we slightly abuse the notation and $1-\rho$ is actually the second smallest eigenvalue of the \emph{normalized} Laplacian matrix for a symmetric $W$.}, which reflects the speed of information mixing over the network. The weight matrix $W$ satisfying \cref{assum1}(b) can be constructed by the Metropolis method \cite{shi2015extra,xiao2004fast}, and \cite{xiao2004fast} discusses how to find an optimal weight matrix with the largest algebraic connectivity. 
	\item \cref{assum1}(c) restricts the variance of local SGs w.r.t. local datasets. Although we do not explicitly assume the unbiasedness of the local SGs, the sampling scheme naturally implies that $\bE[\partial f_i(\vx;\xi_{i,k})]=\eta\nabla f_i(\vx)$. In fact, the theoretical result in this paper holds for any SG satisfying $\bE[\partial f_i(\vx;\xi_{i,k})]=\eta\nabla f_i(\vx)$. Note that many works on SG methods invoke a stronger version of  \cref{assum1}(c) by further assuming $\lambda=0$ to ensure the variance be upper bounded \cite{pu2018distributed2,lian2017can,lan2020first}, while we allow the variance to grow quadratically. In fact, it is safe to set $\lambda=0$ in ERM problems where the cost function $l(\vx;d)$ has bounded gradients for all $\vx,d$, such as the problem of training logistic regression classifiers or training neural networks with cross entropy loss and sigmoid activation function \cite{lan2020first}.
\end{itemize}

The convergence rate of DSGT is quantified by the following term:
\bee\label{eq_r}
R(k)=\frac{\sum_{t=1}^{k}\gamma_t\big(\bE[\|\nabla f(\bar\vx_{t})\|^2]+nL^2\bE[\|X_t-\bone\bar\vx_{t}^\sT\|_\sF^2]\big)}{\sum_{t=1}^k\gamma_t}
\ene
where essentially $\bE[\|\nabla f(\bar\vx_{t})\|^2$ measures the optimality gap to a stationary point and $\bE[\|X_t-\bone\bar\vx_{t}^\sT\|_\sF^2]$ measures the consensus error among nodes.  Note that $R(k)$ is commonly adopted  to  characterize the convergence rate in the centralized non-convex optimization where $\bE[\|X_t-\bone\bar\vx_{t}^\sT\|_\sF^2]$ is obviously equal to $0$ \cite{lan2020first}. Moreover, it holds that $\frac{1}{2n}\sum_{i=1}^n\bE[\|\nabla f(\vx_{i,t})\|^2]\leq\bE[\|\nabla f(\bar\vx_{t})\|^2+nL^2\|X_t-\bone\bar\vx_{t}^\sT\|_\sF^2]$, which implies that $\min_{t\leq k}\frac{1}{2n}\sum_{i=1}^n$$\bE[\|\nabla f(\vx_{i,t})\|^2]\leq R(k)$.

The following theorem is the main theoretical result of this paper, the proof of which is provided in the next section. 
\begin{theorem}\label{theo1}
	Under \cref{assum1}, if the non-increasing stepsizes $\{\gamma_k\}$ in Algorithm \ref{alg_DSGT}  satisfies that $0<\gamma_k\leq\gamma\triangleq\frac{(1-\rho)^2}{\eta L(1+\rho)^2\max\{\sqrt{1+n^2\lambda^2},24n^2\lambda^2\}},\forall k\geq 1$, it holds that
	\bea
	% &R(K)\\
	% &\leq\frac{12n(f(\bar\vx_{1})-f^\star)}{\eta\sum_{t=1}^K\gamma_t}+\frac{8L\sigma_s^2\sum_{t=1}^K\gamma_t^2}{\eta\sum_{t=1}^K\gamma_t}+\frac{2880n\rho^2L\tilde L\sigma_s^2\sum_{t=1}^K\gamma_t^3}{(1-\rho)^3\sum_{t=1}^K\gamma_t}+\frac{48n\rho\gamma L\tilde LC}{(1-\sqrt{\rho})^2\sum_{t=1}^K\gamma_t}
	R(K)\leq\frac{1}{\sum_{t=1}^K\gamma_t}\Big(\frac{9L\sigma_s^2}{\eta}\sum_{t=1}^K\gamma_t^2+\frac{96\rho^2(1+\sqrt{\rho})^2nL\tilde L\sigma_s^2}{(1-\rho)^3}\sum_{t=1}^K\gamma_t^3+12nL\gamma_1C+\frac{12D^2}{\eta L}\Big)
	\ena
where $\tilde L=L\sqrt{1+n^2\lambda^2}$, $C=\frac{2\sqrt{\rho}\tilde L}{1-\sqrt{\rho}}\|X_1-\bone\bar\vx_1^\T\|_\sF^2+\frac{2\tilde L\rho^2(1+\rho)\gamma_1^2}{(1-\sqrt{\rho})^3}\bE[\|Y_1-\bone\bar\vy_1^\T\|_\sF^2]$ and $D=\sqrt{nL(f(\vx_1)-f^\star)}$.
\end{theorem}
\cref{theo1} explicitly shows the dependence of the convergence rate of DSGT on key parameters of the problem. In particular, we consider constant stepsizes and diminishing stepsizes in the following two corollaries. More discussions including the network independence property are given in Section \ref{sec6}.

\begin{corollary}[Constant stepsize]\label{coro2}
	Under \cref{assum1}, let the stepsize $\gamma_k=\gamma_0\in(0,\gamma], \forall k\ge 0$ in Algorithm \ref{alg_DSGT}, where $\gamma$ is defined in \cref{theo1}. We have  
	\bea
	\frac{1}{2nK}\sum_{t=1}^K\sum_{i=1}^n\bE[\|\nabla f(\vx_{i,t})&\|^2]\leq\frac{1}{K}\sum_{t=1}^{K}\Big(\bE[\|\nabla f(\bar\vx_{t})\|^2]+nL^2\bE[\|X_t-\bone\bar\vx_{t}^\sT\|_\sF^2]\Big)\\
	&\leq\frac{12D^2}{K\eta L\gamma_0}+\frac{9L\sigma_s^2\gamma_0}{\eta}+\frac{96\rho^2(1+\sqrt{\rho})^2nL\tilde L\sigma_s^2\gamma_0^2}{(1-\rho)^3}+\frac{12nLC}{K}.
	\ena
	Furthermore, if 
	$
	K\geq\Big(\frac{\eta D(1+\sqrt{\rho})^2}{\sigma_s(1-\rho)^2}\max\big\{\frac{(11\rho^2n+1)\sqrt{1+n^2\lambda^2}}{1-\rho},24n^2\lambda^2\big\}\Big)^2
	$
	and set
	$
	\gamma_k=\frac{D}{\sqrt{K}L\sigma_s}, \forall k\ge 0,
	$  it holds that
	\bea
	\frac{1}{K}\sum_{t=1}^{K}\Big(\bE[\|\nabla f(\bar\vx_{t})\|^2]+nL^2\bE[\|X_t-\bone\bar\vx_{t}^\sT\|_\sF^2]\Big)\leq\frac{30D\sigma_s}{\eta\sqrt{K}}+\frac{12nLC}{K}.
	\ena
\end{corollary}
The proof follows directly from \cref{theo1} and  $\frac{1}{2nk}\sum_{t=1}^k\sum_{i=1}^n\bE[\|\nabla f(\vx_{i,t})\|^2]\leq R(k)$, and hence we omit it here. \cref{coro2} shows that DSGT achieves a rate of order $O(\frac{1}{\sqrt{K}}+\frac{\rho^2}{(1-\rho)^3K})$ for a constant stepsize $\gamma_0=O({1}/{\sqrt{K}})$. As a special case, it also recovers the $O(1/K)$ convergence rate of full gradient descent method by setting $\sigma_s=\lambda=0$. The next corollary characterizes the convergence rate for diminishing stepsizes $\gamma_k=O(1/k^p),\forall p\in[0.5,1)$.

\begin{corollary}[Diminishing stepsize]\label{coro1}
	For any $p\in[0.5,1]$ and $a\in(0,\gamma)$, let $\gamma_k=a/k^p$ in Algorithm \ref{alg_DSGT}. Under \cref{assum1},  it holds that
	\bea\label{eq1_coro1}
	R(k)\leq\frac{1-p}{a(k^{1-p}-1)}\Big(\frac{12D^2}{\eta L}+\frac{9L\sigma_s^2c_2}{\eta}+\frac{96\rho^2(1+\sqrt{\rho})^2nL\tilde L\sigma_s^2c_3}{(1-\rho)^3}+12nLaC\Big)
	\ena
	where $c_2=\frac{2a^2p}{2p-1}$ for $p\in(0.5,1)$, and $c_2= a^2(\ln(k)+1)$ for $p=0.5$; $c_3=\frac{3a^3p}{3p-1}$.
\end{corollary}
\begin{proof}
	The result follows from \cref{theo1} by using $\int_{1}^k t^{-p}dt\leq\sum_{t=1}^k t^{-p}\leq\int_{1}^k t^{-p}dt+1,\forall p>0$, $\int_{1}^k t^{-p}dt=\frac{k^{1-p}-1}{1-p}$ for $p>0,p\neq 1$ and $\int_{1}^k t^{-p}dt=\ln(k)$ for $p=1$.
\end{proof}

\cref{coro1} shows that DSGT converges at a rate of $O(1/k^{1-p})$ (up to a logarithm factor if $p=0.5$) if $\gamma_k=O(1/k^p)$ for any $p\in[0.5,1)$. In contrast, the existing decentralized algorithms (e.g. \cite{lian2017can,assran2018stochastic,lu2019gnsd}) do not report convergence results for diminishing stepsizes, and \cite{pu2018distributed2} analyzes the convergence rate only for $\gamma_k=O(1/k)$.

We finally provide a non-asymptotic result of DSGT for general convex objective functions. To this end, let 
\bee
R_c(k)=\frac{\sum_{t=1}^k\gamma_t\big(\bE[f(\bar\vx_{t})]-f^\star+\frac{L}{2}\bE[\|X_t-\bone\bar\vx_t^\T\|_\sF^2]\big)}{\sum_{t=1}^k\gamma_t}.
\ene

\begin{theorem}[The convex case]\label{theo2}
	Under the conditions in \cref{theo1}, if all $f_i,i\in\cV$ are convex and $f$ has a minimizer $\vx^\star$, i.e., $f(\vx^\star)=f^\star$, then we have
	\bea
	% &\sum_{t=1}^k\gamma_t\Big(\bE[f(\bar\vx_{t})]-f^\star+L\|X_t-\bone\bar\vx_t^\T\|_\sF^2\Big)\\
	R_c(K)\leq\frac{1}{\sum_{t=1}^K\gamma_t}\Big(\frac{3n\|\bar\vx_1-\vx^\star\|}{\eta}+\frac{5\sigma_s^2}{\eta n}\sum_{t=1}^K\gamma_t^2+\frac{48\rho^2(1+\sqrt{\rho})^2\tilde L\sigma_s^2}{(1-\rho)^3}\sum_{t=1}^K\gamma_t^3+6\gamma_1C\Big).
	\ena
\end{theorem}

The proof is deferred to Section \ref{sec_pfthm}. Note that \cref{theo2} cannot be trivially derived from \cref{theo1} and the convexity of $f_i,i\in\cV$. Similarly, we can characterize the convergence rates for constant and diminishing stepsizes in the convex case, but we omit it here for saving space.

\section{Proof of Theorem \ref{theo1} and \ref{theo2}}
\label{sec_pfthm}

This section is mainly devoted to the proof of Theorems \ref{theo1} and \ref{theo2}.  

Our main idea lies in the use of two important quantities $\sum_{t=1}^{K}\gamma_t\bE[\|X_{t}-\bone\bar\vx_{t}^\T\|_\sF^2]$ and $\sum_{t=1}^K\gamma_t^2\bE[\|\bar\vy_t\|^2]$, which respectively measure the cumulative consensus error and the optimality gap. We bound them  by each other via two interlacing linear inequalities, which are regarded as two linear constraints of a linear programming (LP) problem. Then, $R(K)$ is bounded via a  positive linear combination of those two quantities, which is considered as the objective function of the LP problem. By explicitly solving the LP problem, a good upper bound of $R(K)$ is derived.  In short, our theoretical results are derived by using a LP based approach. To this end, we begin with several lemmas for deriving the linear inequality constraints.

\subsection{Preliminary lemmas}

The following lemma characterizes the consensus errors of $X_{k+1}$ and $Y_{k+1}$.

\begin{lemma}\label{lemma1}
	Let $0\leq\gamma_k\leq\frac{(1-\rho)^2}{\eta\tilde L(1+\rho)^2},\forall k$ and \cref{assum1} hold. We have
	\bea\label{eq1_s2}
	&\bE[\|X_{k+1}-\bone\bar\vx_{k+1}^\T\|_\sF^2]\leq \frac{2\rho^2}{1+\rho^2}\bE[\|X_k-\bone\bar\vx_k^\T\|_\sF^2]+\frac{2\rho^2\gamma_k^2}{1-\rho^2}\bE[\|Y_k-\bone\vy_k^\T\|_\sF^2],
	\ena
	and
\bea\label{eq1_s1}
	&\bE[\|Y_{k+1}-\bone\bar\vy_{k+1}^\T\|_\sF^2]\leq4\sigma_s^2+\frac{2n(\gamma_k\eta \tilde L)^2}{1-\rho^2}\bE[\|\bar\vy_k\|^2]+6(\lambda\eta)^2n\bE[\|\nabla f(\bar\vx_k)\|^2]\\
	&\quad+\Big(\frac{2}{1+\rho^2}+\frac{4(\gamma_k\eta \tilde L)^2}{1-\rho^2}\Big)\rho^2\bE[\|Y_{k}-\bone\bar\vy_k^\T\|_\sF^2]+\frac{4(\eta \tilde L(1+\rho))^2}{1-\rho^2}\bE[\|X_{k}-\bone\bar\vx_k^\T\|_\sF^2].
	\ena
\end{lemma}

\begin{proof}
	See \cref{adx0}.
\end{proof}

\begin{remark}
	The proof of \cref{lemma1} mainly relies on the contraction property of the mixing matrix $W$, i.e., $\|WX-\bone\bar\vx\|\leq\rho\|X-\bone\bar\vx\|$ where $\bar\vx=\frac{1}{n}\bone\bone^\T$. We note that similar results has been reported in \cite{pu2018distributed2}  and \cite{qu2017harnessing}. However, Eq. (18) in \cite{qu2017harnessing} focuses only on convex case and considers full gradient rather than SG methods. \cite[Lemma 4]{pu2018distributed2} studies only strongly convex functions and assumes $\lambda=0$. Moreover, we carefully optimized the coefficients in \cref{eq1_s2} and \cref{eq1_s1}.
\end{remark}

The upper bound of $\bE[\|X_{k+1}-\bone\bar\vx_{k+1}^\T\|_\sF^2]$ in \cref{lemma1} depends on $\bE[\|Y_k-\bone\vy_k^\T\|_\sF^2]$. The following lemma removes this dependence by adopting an eigendecomposition method.

\begin{lemma}\label{lemma2}
	Let the conditions in \cref{lemma1} hold and $\{\gamma_k\}$ be non-increasing. Let
	$
	\theta\triangleq\frac{2\rho^2}{1+\rho^2}+\frac{2(\rho\gamma\eta\tilde L)^2}{1-\rho^2}+\frac{2\rho\gamma\eta\tilde L\sqrt{(\rho\gamma\eta\tilde L)^2+2(1+\rho)^2}}{1-\rho^2}.
	$
	We have $\theta\leq\frac{\sqrt{\rho}+\rho}{1+\rho}<1$ and
	\bea\label{eq32_s3}
% &\bE[\|X_{k+1}-\bone\bar\vx_{k+1}^\T\|_\sF^2]\leq C_{0}(k)+\sum_{t=1}^{k}(k-t)\theta^{k-t}\frac{2\rho^2\gamma_t^2}{1-\rho^2}\Big(\frac{2n(\eta \tilde L \gamma_t)^2}{1-\rho^2}\bE[\|\bar\vy_{t}\|^2]+4\sigma_s^2+6n(\eta\lambda)^2\bE[\|\nabla f(\bar\vx_t)\|^2]\Big)\\
&\bE[\|X_{k+1}-\bone\bar\vx_{k+1}^\T\|_\sF^2]\\
&\leq C_{0}(k)+\sum_{t=1}^{k}(k-t)\theta^{k-t-1}\gamma_t^2\Big(C_1\gamma_t^2\bE[\|\bar\vy_{t}\|^2]+C_2+C_3\bE[\|\nabla f(\bar\vx_t)\|^2]\Big)
\ena
	where
	\bea\label{eq4_s3}
	&C_0(k)=\theta^k\Big(\|X_1-\bone\bar\vx_1^\T\|_\sF^2+\frac{2\rho^2\gamma_t^2k}{(1-\rho^2)\theta}\bE[\|Y_1-\bone\bar\vy_1^\T\|_\sF^2]\Big),\\
	&C_1=\frac{4n\rho^2(\eta \tilde L )^2}{(1-\rho^2)^2},\ C_2=\frac{8\rho^2\sigma_s^2}{1-\rho^2},\ C_3=\frac{12n\rho^2(\eta\lambda)^2}{1-\rho^2}.
	\ena
\end{lemma}

\begin{proof}
It follows from \cref{lemma1} that for all $k\geq1$,
\bea\label{eq_lmi}
\underbrace{\begin{bmatrix}
		\bE[\|X_{k+1}-\bone\bar\vx_{k+1}^\T\|_\sF^2] \\
		\bE[\|Y_{k+1}-\bone\bar\vy_{k+1}^\T\|_\sF^2]
\end{bmatrix}}_{\textstyle \triangleq \vz_{k+1}}&\preccurlyeq
\underbrace{\begin{bmatrix}
	P_k^{11}            & P_k^{12}          \\
	P_k^{21} & P_k^{22}
\end{bmatrix}}_{\textstyle \triangleq P_k}
\underbrace{\begin{bmatrix}
		\bE[\|X_{k}-\bone\bar\vx_{k}^\T\|_\sF^2] \\
		\bE[\|Y_{k}-\bone\bar\vy_{k}^\T\|_\sF^2]
\end{bmatrix}}_{\textstyle \triangleq \vz_{k}}\\
&\!\! +\underbrace{\begin{bmatrix}
		0 \\
		\frac{2n(\eta \tilde L \gamma_k)^2}{1-\rho^2}\bE[\|\bar\vy_{k}\|^2]+6n(\eta\lambda)^2\bE[\|\nabla f(\bar\vx_k)\|^2]+4\sigma_s^2
\end{bmatrix}}_{\textstyle \triangleq \vu_{k}}
\ena
where $\preccurlyeq$ denotes the element-wise less than or equal sign, and
\bea\label{eq_p}
P_k^{11}=\frac{2\rho^2}{1+\rho^2},\ P_k^{12}=\frac{2\rho^2\gamma_k^2}{1-\rho^2},\ 
P_k^{21}=\frac{4(\eta \tilde L(1+\rho))^2}{1-\rho^2},\ P_k^{22}=\frac{2\rho^2}{1+\rho^2}+\frac{4(\rho\gamma_k\eta \tilde L)^2}{1-\rho^2}
\ena

Clearly, we have $P_k\preccurlyeq P_t$ and $\prod_{i=t}^k P_i=P_kP_{k-1}\cdots P_t\preccurlyeq P_t^{k-t+1},\forall k\geq t\geq 1$ since $P_k$ is nonnegative and $\{\gamma_k\}$ is non-increaseing. It then follows from the above element-wise linear matrix inequality that
\bee\label{eq1_s3}
\vz_{k+1}\preccurlyeq \prod_{i=1}^k P_i\vz_1+\sum_{t=1}^{k}\prod_{i=t+1}^{k} P_{i}\vu_{t}\preccurlyeq P_1^{k}\vz_1+\sum_{t=1}^{k}P_{t}^{k-t}\vu_{t},\ \forall k\geq1.
\ene

We need to bound $P_1^{k}$ and $P_{t}^{k-t}$. To this end, consider the eigendecomposition $P_t=T_t \Lambda_t T_t^{-1}$ with $\Lambda_t=\diag(\underline{\theta}_t,\theta_t)$, where $\theta_t$ and $\underline{\theta}_t$ are the two eigenvalues of $P_t$ and $|\theta_t|\geq|\underline{\theta}_t|$. Let $\Psi_t\triangleq\sqrt{(P_t^{11}-P_t^{22})^2+4P_t^{12}P_t^{21}}$. With some tedious calculations, we have
\bea\label{eq_theta}
\centering
&\theta_t=\frac{P_t^{11}+P_t^{22}+\Psi_t}{2}=\frac{2\rho^2}{1+\rho^2}+\frac{2(\rho\gamma_t\eta\tilde L)^2}{1-\rho^2}+\frac{2\rho\gamma_t\eta\tilde L\sqrt{(\rho\gamma_t\eta\tilde L)^2+2(1+\rho)^2}}{1-\rho^2},\\
&\underline{\theta}_t=\frac{P_t^{11}+P_t^{22}-\Psi_t}{2},\\
&T_t=\begin{bmatrix}
	\frac{P_t^{11}-P_t^{22}-\Psi_t}{2P_t^{21}} & \frac{P_t^{11}-P_t^{22}+\Psi_t}{2P_t^{21}} \\
	1                                                                                                      & 1
\end{bmatrix},\
T_t^{-1}=\begin{bmatrix}
	-\frac{P_t^{21}}{\Psi_t} & \frac{P_t^{11}-P_t^{22}+\Psi_t}{2\Psi_t} \\	\frac{P_t^{21}}{\Psi_t} & \frac{P_t^{22}-P_t^{11}+\Psi_t}{2\Psi_t}
\end{bmatrix}, 
\ena
Hence, for any $k\geq0$, we have
\bee\label{eq5_s3}
P_t^k=T_t \Lambda_t^k T_t^{-1}\preccurlyeq\begin{bmatrix}
	\frac{\underline{\theta}_t^k+\theta_t^k}{2}+\frac{\left(P_t^{11}-P_t^{22}\right)\left(\theta_t^k-\underline{\theta}_t^k\right)}{2 \Psi_t} & \frac{P_t^{12}}{\Psi_t}\left(\theta_t^k-\underline{\theta}_t^k\right) & \\
	\frac{P_t^{21}}{\Psi_t}\left(\theta_t^k-\underline{\theta}_t^k\right) & \frac{\underline{\theta}_t^k+\theta_t^k}{2}+\frac{\left(P_t^{11}-P_t^{22}\right)\left(\underline{\theta}_t^k-\theta_t^k\right)}{2 \Psi_t} 
	\end{bmatrix}
\ene
Note that $-\theta\leq\underline{\theta}_t\leq\theta_t\leq\theta<1,\forall t$. In fact, $\theta$ is obtained by replacing  $\gamma_t$ in $\theta_t$ with its upper-bound $\gamma$, and the relation follows from that $\theta_t$ is increasing with $\gamma_t$.

Let $P_1^k\vz_1[1,:]$ be the first row of $P_1^k\vz_1$. It follows from \cref{eq_lmi} and \cref{eq5_s3}  that
\bea\label{eq2_s3}
&P_1^k\vz_1[1,:]\preccurlyeq \theta^k(\bE[\|X_1-\bone\bar\vx_1^\T\|_\sF^2]+P_1^{12}k/\theta\bE[\|Y_1-\bone\bar\vy_1^\T\|_\sF^2]) 
\ena
where we used $\underline{\theta}_t\leq\theta_t\leq\theta,\forall t$ and $\theta_t^k-\underline{\theta}_t^k=(\theta_t-\underline{\theta}_t)\sum_{l=0}^{k-1}\theta_t^l\underline{\theta}_t^{k-1-l}=\Psi_tk\theta_t^{k-1}$.

Let $P_{t}^{k-t}\vu_{t}[1]$ be the first element of $P_{t}^{k-t}\vu_{t}$. Similarly, it follows from \cref{eq5_s3} and \cref{eq_lmi} that
\bea\label{eq21_s3}
P_{t}^{k-t}\vu_{t}[1]\leq (k-t)\theta^{k-t-1}P_t^{12}\Big(\frac{2n(\eta \tilde L \gamma_k)^2}{1-\rho^2}\bE[\|\bar\vy_{k}\|^2]+6n(\eta\lambda)^2\bE[\|\nabla f(\bar\vx_k)\|^2]+4\sigma_s^2\Big).
\ena
The desired result \cref{eq32_s3} then follows from the combination of \cref{eq1_s3}, \cref{eq2_s3}, \cref{eq21_s3} and the definition of $P_t^{12}$ in \cref{eq_p}.

It remains to prove that $\theta\leq\frac{\sqrt{\rho}+\rho}{1+\rho}$. Recall that $\gamma\eta \tilde L\leq\frac{(1-\rho)^2}{(1+\rho)^2}$, and substituting the upper-bound into $\theta$ yields a polynomial fraction in $\rho$. Thus, showing that $\theta\leq\frac{\sqrt{\rho}+\rho}{1+\rho}$ can be transformed into checking the positiveness of some polynomial w.r.t. $\rho$ over $\rho\in(0,1)$, which can be easily solved by finding the minimum value or simply plotting the curve. We omit the details here to save space.
\end{proof}

We are ready to provide the first linear inequality constraint on $\sum_{t=1}^{K}\gamma_t\bE[\|X_{t}-\bone\bar\vx_{t}^\T\|_\sF^2]$ and $\sum_{t=1}^K\gamma_t^2\bE[\|\bar\vy_t\|^2]$, which directly follows from \cref{lemma2}.

\begin{lemma}\label{lemma3}
	Let the conditions in \cref{lemma1} hold and $\{\gamma_k\}$ be non-increasing. We have
	\bea\label{eq7_s4}
&\sum_{t=1}^{k}\gamma_t\bE[\|X_{t}-\bone\bar\vx_{t}^\T\|_\sF^2]\\
&\leq \tilde C_0+\frac{1}{(1-\theta)^2}\sum_{t=1}^{k}\gamma_t^3\Big(C_1\gamma_t^2\bE[\|\bar\vy_{t}\|^2]+C_2+C_3\bE[\|\nabla f(\bar\vx_t)\|^2]\Big)
\ena
where $\theta, C_1,C_2$ and $C_3$ are as defined in \cref{lemma2}, and
\bea
\tilde C_0=\frac{\gamma_1\theta}{1-\theta}\|X_1-\bone\bar\vx_1^\T\|_\sF^2+\frac{2\rho^2\gamma_1^3}{(1-\rho^2)(1-\theta)^2}\bE[\|Y_1-\bone\bar\vy_1^\T\|_\sF^2].
\ena
\end{lemma}

Note that $\gamma_t^3 C_1$ is upper-bounded by the constant $\gamma^3 C_1$. The proof relies on the following lemma.

\begin{lemma}\label{lemma_gamma}
	Let $\{s_k\}$ be a nonnegative sequence, $\theta$ be a constant in $(0,1)$, and $a_k=\sum_{t=1}^k s_t(k-t)\theta^{k-t-1}$. It holds that
	\bee
	\sum_{t=1}^k a_t\leq\frac{1}{(1-\theta)^2}\sum_{t=1}^k s_t,\ \forall \theta\in(0,1),k\in\bN.
	\ene
\end{lemma}
\begin{proof}
	Let $b_k=\sum_{t=1}^k s_t\theta^{k-t}$. Taking derivative w.r.t. $\theta$ implies that $\frac{db_k}{d\theta}=a_k$. Moreover, set $b_0=0$ and we have $b_{k}=\theta b_{k-1}+s_k,\forall k\geq 1$. Summing this equality over $k=1,2,\dots$ implies that
	\bee
	\sum_{t=1}^k b_t=\theta\sum_{t=1}^{k} b_{t-1}+\sum_{t=1}^k s_t\leq\theta\sum_{t=1}^k b_t+\sum_{t=1}^k s_t.
	\ene
	Hence, we have $\sum_{t=1}^k b_t\leq\frac{1}{1-\theta}\sum_{t=1}^k s_t$. The result then follows from $\sum_{t=1}^k a_t=\frac{d}{d\theta}(\sum_{t=1}^k b_t)$.
\end{proof}

\begin{proof}[Proof of \cref{lemma3}]
The result directly follows from \cref{lemma2} and \cref{lemma_gamma}. More specifically, let $s_k$ and $a_k$ be $\gamma_k^3(C_1\gamma_k^2\bE[\|\bar\vy_k\|^2]+C_2+C_3\bE[\|\nabla f(\bar\vx_k)\|^2])$ and $a_k=\sum_{t=1}^k s_t(k-t)\theta^{k-t-1}$, respectively. We have from \cref{lemma2} that $\gamma_t\bE[\|X_{t}-\bone\bar\vx_{t}^\T\|_\sF^2]\leq a_t+ C_0(t)$. The desired result then follows from \cref{lemma_gamma} and the relation $\sum_{t=1}^\infty \theta^t=\frac{\theta}{1-\theta}$ and $\sum_{t=1}^\infty t\theta^t=\frac{1}{(1-\theta)^2}$.
\end{proof}

Finally, we provide the last lemma to introduce the second linear constraint. 
\begin{lemma}\label{lemma4}
	Let the conditions in \cref{lemma1} hold and $\{\gamma_k\}$ be non-increasing. We have
\bea\label{eq4_s5}
\sum_{t=1}^k\gamma_t^2\bE[\|\bar\vy_t\|^2]&\leq\frac{4\eta\gamma(f(\bar\vx_1)-f^\star)}{n(4-3\eta \tilde L\gamma)}+\frac{4\eta \tilde L}{n(4-3\eta \tilde L\gamma)}\sum_{t=1}^k\gamma_t\bE[\|X_t-\bone\vx_t^\T\|_\sF^2]\\
&\quad+\frac{4\lambda^2\eta^2}{4-3\eta \tilde L\gamma}\sum_{t=1}^k\gamma_t^2\bE[\|\nabla f(\bar\vx_t)\|^2]+\frac{4\sigma_s^2}{n^2(4-3\eta \tilde L\gamma)}\sum_{t=1}^k\gamma_t^2.
\ena
\end{lemma}
\begin{proof}
	Recall that $f$ is Lipschitz smooth with parameter $nL$. We have
\bea\label{eq1_lemma4}
&f(\bar\vx_{k+1})\\
&=f(\bar\vx_k-\gamma_k\bar\vy_k)\leq f(\bar\vx_k)-\gamma_k\bar\vy_k^\T\nabla f(\bar\vx_k)+\frac{nL\gamma_k^2}{2}\|\bar\vy_k\|^2\\
&\leq f(\bar\vx_k)-\gamma_k \bar\vy_k^\T(\nabla f(\bar\vx_k)-\nabla F(X_k)^\T\bone)-\gamma_k n\bar\vy_k^\T\vg_k+\frac{nL\gamma_k^2}{2}\|\bar\vy_k\|^2\\
&\leq f(\bar\vx_k)-\gamma_k \bar\vy_k^\T(\nabla F(\bone\vx_k^\T)-\nabla F(X_k))^\T\bone-\frac{\gamma_k n}{\eta}\bar\vy_k^\T(\eta\vg_k-\bar\vy_k)-n(\frac{\gamma_k}{\eta}-\frac{L\gamma_k^2}{2})\|\bar\vy_k\|^2
\ena
We first bound $-\gamma_k \bar\vy_k^\T(\nabla F(\bone\vx_k^\T)-\nabla F(X_k))^\T\bone$. It holds that
\bea\label{eq2_lemma4}
-\gamma_k \bar\vy_k^\T(\nabla F(\bone\vx_k^\T)-\nabla F(X_k))^\T\bone&\leq\frac{n\tilde L\gamma_k^2}{4}\|\bar\vy_k\|^2+\frac{1}{\tilde L}\|\nabla F(\bone\vx_k^\T)-\nabla F(X_k)\|_\sF^2\\
&\leq\frac{n\tilde L\gamma_k^2}{4}\|\bar\vy_k\|^2+\frac{L^2}{\tilde L}\|X_k-\bone\bar\vx_k^\T\|_\sF^2
\ena
where we used the Lipschitz smoothness and the relation $2\va^\T\vb\leq c\|\va\|^2+\frac{1}{c}\|\vb\|^2$ (set $c=n\tilde L/2$).

Then, we bound $\bE[-\bar\vy_k^\T(\eta\vg_k-\bar\vy_k)|\cF_{k-1}]$. Recall that $\bE[\bar\vy_k|\cF_{k-1}]=\eta\vg_k$. We have
\bea\label{eq3_lemma4}
&\bE[-\bar\vy_k^\T(\eta\vg_k-\bar\vy_k)|\cF_{k-1}]=\bE[(\eta\vg_k-\bar\vy_k)^\T(\eta\vg_k-\bar\vy_k)|\cF_{k-1}]\\
&=\frac{1}{n^2}\bE[\|\partial f_1(\vx_{1,k};\xi_{1,k})-\eta\nabla f_1(\vx_{1,k})+\ldots+\partial f_n(\vx_{n,k};\xi_{n,k})-\eta\nabla f_n(\vx_{n,k})\|^2|\cF_{k-1}]\\
&=\frac{1}{n^2}\sum_{i=1}^n\bE[\|\partial f_i(\vx_{i,k};\xi_{i,k})-\eta\nabla f_i(\vx_{i,k})\|^2|\cF_{k-1}]\leq\frac{\sigma_s^2}{n^2}+\frac{\lambda^2}{n^2}\|\eta\nabla \vf(X_k)\|_\sF^2\\
&\leq\frac{\sigma_s^2}{n^2}+\frac{\lambda^2\eta^2}{n^2}\big(\|\nabla \vf(X_k)-\nabla\vf(\bone\bar\vx_k^\T)\|_\sF+\|\nabla\vf(\bone\bar\vx_k^\T)\|_\sF\big)^2\\
&\leq\frac{\sigma_s^2}{n^2}+2\lambda^2\eta^2L^2\|X_k-\bone\vx_k^\T\|_\sF^2+\frac{2\lambda^2\eta^2}{n}\|\nabla f(\bar\vx_k)\|^2\\
&\leq\frac{\sigma_s^2}{n^2}+\frac{n\lambda^2 L^2\eta}{\gamma_k\tilde L}\|X_k-\bone\vx_k^\T\|_\sF^2+\lambda^2\eta^2\|\nabla f(\bar\vx_k)\|^2
\ena
where we used the fact that $\xi_{i,k}$ and $\xi_{j,k}$ are independent if $i\neq j$ to obtain the last equality. The last inequality follows from  $\eta\gamma_k\tilde L\leq 1$ and $n\geq 2$ (when $n=1$ we have $X_k=\bone\bar\vx_k^\T$ and hence the inequality trivially holds).

Plug \cref{eq2_lemma4} and \cref{eq3_lemma4} into \cref{eq1_lemma4} and take the full expectations on both sides of it. We obtain
\bea
\bE[f(\bar\vx_{k+1})]-f^\star&\leq\bE[f(\bar\vx_k)]-f^\star-n\gamma_k(\frac{1}{\eta}-\frac{3\tilde L\gamma_k}{4})\bE[\|\bar\vy_k\|^2]\\
&\quad+\tilde L\bE[\|X_k-\bone\vx_k^\T\|_\sF^2]+\frac{\gamma_k\sigma_s^2}{\eta n}+n\lambda^2\eta\gamma_k\bE[\|\nabla f(\bar\vx_k)\|^2]
\ena
Multiply both sides by $\gamma_k$. We obtain
\bea\label{eq4_lemma4}
\gamma_k(\bE[f(\bar\vx_{k+1})]-f^\star)&\leq\gamma_k(\bE[f(\bar\vx_k)]-f^\star)-n\gamma_k^2(\frac{1}{\eta}-\frac{3\tilde L\gamma_k}{4})\bE[\|\bar\vy_k\|^2]\\
&\quad+\tilde L\gamma_k\bE[\|X_k-\bone\vx_k^\T\|_\sF^2]+\frac{\gamma_k^2\sigma_s^2}{\eta n}+n\lambda^2\eta\gamma_k^2\bE[\|\nabla f(\bar\vx_k)\|^2]
\ena
Summing up \cref{eq4_lemma4} over $k=1,2,\ldots$, and using that $\gamma_k$ is non-increasing, we have
\bea\label{eq5_s4}
\sum_{t=1}^k\gamma_t^2(4-3\eta \tilde L\gamma_t)\bE[\|\bar\vy_t\|^2]&\leq\frac{4\eta\gamma(\bE[f(\bar\vx_1)]-f^\star)}{n}+\frac{4\eta \tilde L}{n}\sum_{t=1}^k\gamma_t\bE[\|X_t-\bone\vx_t^\T\|_\sF^2]\\
&\quad+{4\lambda^2\eta^2}\sum_{t=1}^k\gamma_t^2\bE[\|\nabla f(\bar\vx_t)\|^2]+\frac{4\sigma_s^2}{n^2}\sum_{t=1}^k\gamma_t^2.
\ena
The desired result follows by dividing both sides by $4-3\eta \tilde L\gamma$.
\end{proof}

\subsection{Proof of \cref{theo1}}\label{sec42}

It follows from the Lipschitz smoothness of $f$ (with parameter $nL$) that
$
f(\bar\vx_{k+1})\leq f(\bar\vx_{k})-\gamma_k\bar\vy_k^\T\nabla f(\bar\vx_{k})+\frac{nL\gamma_k^2}{2}\|\bar\vy_k\|^2,
$
which is followed by
\bee\label{eq12_s5}
\bE[f(\bar\vx_{k+1})]\leq \bE[f(\bar\vx_{k})]-\frac{\gamma_k\eta}{n}\bE[n\vg_k^\T\nabla f(\bar\vx_{k})]+\frac{nL\gamma_k^2}{2}\bE[\|\bar\vy_k\|]^2.
\ene

Notice that
\bea\label{eq13_s5}
-n\vg_k^\T\nabla f(\bar\vx_{k})&=(\nabla f(\bar\vx_k)-n\vg_k)^\T\nabla f(\bar\vx_k)-\|\nabla f(\bar\vx_{k})\|^2\\
% &\leq-\|\nabla f(\bar\vx_{k})\|^2+\|\nabla f(\bar\vx_{k})-n\vg_k\|_\sF\|\nabla f(\bar\vx_{k})\|\\
&\leq-\|\nabla f(\bar\vx_{k})\|^2+\sqrt{n}L\|X_k-\bone\bar\vx_{k}^\sT\|_\sF\|\nabla f(\bar\vx_{k})\|\\
&\leq-\frac{1}{2}\|\nabla f(\bar\vx_{k})\|^2+\frac{nL^2}{2}\|X_k-\bone\bar\vx_{k}^\sT\|_\sF^2
\ena
where we used $\nabla f(\bar\vx_k)=\nabla F(\bone\bar\vx_k^\T)^\T\bone$. This inequality combined with \cref{eq12_s5} yields
\bee
\bE[f(\bar\vx_{k+1})]\leq \bE[f(\bar\vx_{k})]-\frac{\gamma_k\eta}{2n}\bE[\|\nabla f(\bar\vx_{k})\|^2]+\frac{nL\gamma_k^2}{2}\bE[\|\bar\vy_k\|]^2+\frac{\gamma_k\eta L^2}{2}\bE[\|X_k-\bone\bar\vx_{k}^\sT\|_\sF^2].
\ene

Summing this inequality over $k=1,2\ldots,K$ gives
\bea\label{eq14_s5}
&\sum_{t=1}^{K}\Big(\gamma_t\bE[\|\nabla f(\bar\vx_{t})\|^2]+nL^2\gamma_t\bE[\|X_t-\bone\bar\vx_{t}^\sT\|_\sF^2]\Big)\\
&\leq \frac{2n(f(\bar\vx_{1})-f^\star)}{\eta}+\underbrace{\frac{Ln^2}{\eta}\sum_{t=1}^{K}\gamma_t^2\bE[\|\bar\vy_t\|]^2+2nL^2\sum_{t=1}^{K}\gamma_t\bE[\|X_t-\bone\bar\vx_{t}^\sT\|_\sF^2]}_{\textstyle\psi}
\ena

The rest of the proof is devoted to bound $\psi$ in \cref{eq14_s5}. This is achieved by explicitly solving a LP problem, where we treat $\sum_{t=1}^{K}\gamma_t\bE[\|X_t-\bone\bar\vx_{t}^\sT\|_\sF^2]$ and $\sum_{t=1}^{K}\gamma_t^2\bE[\|\bar\vy_t\|]^2$ as the two decision variables that must satisfy two linear constraints given by \cref{lemma3} and \cref{lemma4}.

More specifically, we have 
\bee\label{eq_psi}
\psi\leq p^\star
\ene
where $p^\star$ is the optimal value of the following linear programming problem (LP):
\bea\label{linear}
&\maximize_{\tau_1>0,\tau_2>0}&&c_1\tau_1+c_2\tau_2\\
&\st &&\tau_1\leq a_1\tau_2+b_1,\quad \tau_2\leq a_2\tau_1+b_2
\ena
where $c_1,c_2,a_1,a_2,b_1,b_2$ are positive coefficients inheriting from $\psi$, \cref{eq7_s4} and \cref{eq4_s5}, and are defined as follows:
\bea\label{coefficients}
c_1&\triangleq2nL^2,\ c_2\triangleq\frac{Ln^2}{\eta},\ a_1\triangleq\frac{C_1\gamma_t^3}{(1-\theta)^2}=\frac{4n\rho^2(\eta\tilde L)^2\gamma_t^3}{(1-\theta)^2(1-\rho^2)^2},\ a_2\triangleq\frac{4\eta\tilde L}{n(4-3\eta \tilde L\gamma)},\\ 
b_1&\triangleq\frac{8\rho^2\sigma_s^2}{(1-\theta)^2(1-\rho^2)}\sum_{t=1}^K\gamma_t^3+\frac{12n\rho^2(\eta\lambda)^2}{(1-\theta)^2(1-\rho^2)}\sum_{t=1}^K\gamma_t^3\bE[\|\nabla f(\bar\vx_t)\|^2]+\tilde C_0,\\
b_2&\triangleq\frac{4\lambda^2\eta^2}{4-3\eta \tilde L\gamma}\sum_{t=1}^K\gamma_t^2\bE[\|\nabla f(\bar\vx_t)\|^2]+\frac{4\sigma_s^2}{n^2(4-3\eta \tilde L\gamma)}\sum_{t=1}^K\gamma_t^2+\frac{4\eta\gamma(f(\bar\vx_1)-f^\star)}{n(4-3\eta \tilde L\gamma)}.
\ena
In fact, the first constraint is from \cref{lemma3} and the second constraint is from \cref{lemma4}. Since \cref{linear} is a simple LP with only two variables, it can be readily checked (e.g. using graphical methods) that if $a_1a_2<1$, then the optimal solution is the point where both constraints are active. That is, $p^\star=c_1\tau_1^\star+c_2\tau_2^\star$, where $\tau_1^\star= a_1\tau_2^\star+b_1$ and $\tau_2^\star= a_2\tau_1^\star+b_2$. Therefore, we have
\bea\label{eq_pstar1}
p^\star&=[c_1,c_2]\begin{bmatrix}
	1 & -a_1\\
	-a_2 & 1
\end{bmatrix}^{-1}
\begin{bmatrix}
	b_1\\
	b_2
\end{bmatrix}=\frac{1}{1-a_1a_2}
[c_1,c_2]\begin{bmatrix}
	1 & a_1\\
	a_2 & 1
\end{bmatrix}
\begin{bmatrix}
	b_1\\
	b_2
\end{bmatrix}\\
&=\frac{c_1b_1+c_2a_2b_1+c_1a_1b_2+c_2b_2}{1-a_1a_2}.
\ena 

We now show that $a_1a_2<1$ is indeed true. In fact, we have the following relations which will be used in subsequent proofs.
\bee\label{eq_misc}
\frac{1}{1-\theta}\leq\frac{1+\rho}{1-\sqrt{\rho}},\  4-3\eta \tilde L \gamma\geq(1+\rho)^2,\ a_1a_2\leq\frac{1}{4},\forall \rho\in[0,1).
\ene
The first relation follows from \cref{lemma2} that $\theta\leq\frac{\sqrt{\rho}+\rho}{1+\rho}$. The second inequality follows from $\eta\tilde L\gamma\leq\frac{(1-\rho)^2}{(1+\rho)^2}$. The last one is from
\bee
a_1a_2=\frac{16\rho^2(\eta\tilde L\gamma)^3}{(1-\theta)^2(1-\rho^2)^2(4-3\eta\tilde\gamma)}\leq\frac{16\rho^2(1-\rho)^2(1+\sqrt{\rho})^2}{(1+\rho)^8}\leq\frac{32\rho^2(1-\rho)^2}{(1+\rho)^7}<\frac{1}{4}
\ene
where we used $\eta\tilde L\gamma\leq\frac{(1-\rho)^2}{(1+\rho)^2}$, $\frac{1}{1-\theta}\leq\frac{1+\rho}{1-\sqrt{\rho}}$ and $4-3\eta \tilde L \gamma\geq(1+\rho)^2$ to obtain the first inequality. The second inequality used the relation $(1+\sqrt{\rho})^2\leq2(1+\rho)$, and the last inequality can be readily checked by e.g. finding the maximum on $\rho\in(0,1)$.

Then, we have from \cref{coefficients} and \cref{eq_misc} that
\bea
&c_1b_1+c_2a_2b_1=(c_1+c_2a_2)b_1\\
&\leq6nL\tilde L(1+\rho)\Big(\frac{8\rho^2\sigma_s^2\sum_{t=1}^K\gamma_t^3+12n(\rho\eta\lambda)^2\sum_{t=1}^K\gamma_t^3\bE[\|\nabla f(\bar\vx_t)\|^2]}{(1-\rho^2)(1-\sqrt{\rho})^2}+\frac{\tilde C_0}{(1+\rho)^2}\Big)
\ena
where we used $2+(1+\rho)^2\leq3(1+\rho)$. Moreover,
\bea
&c_1a_1b_2+c_2b_2\leq\frac{Ln^2b_1}{\eta}\big(\frac{(1+\rho)^2}{8}+1\big)\\
&\leq\frac{9Ln}{2\eta}\Big(n\lambda^2\eta^2\sum_{t=1}^K\gamma_t^2\bE[\|\nabla f(\bar\vx_t)\|^2]+\frac{\sigma_s^2}{n}\sum_{t=1}^K\gamma_t^2+{\eta\gamma(f(\bar\vx_1)-f^\star)}\Big)
\ena
where the first inequality follows from $\frac{c_1a_1}{c_2}\leq \frac{a_1a_2(1+\rho)^2}{2}\leq\frac{(1+\rho)^2}{8}$.

Therefore, it follows from \cref{eq_pstar1} that
\bea\label{eq_pstar}
p^\star&\leq \frac{4}{3}\bigg(\Big(\frac{72\rho^2(\eta Ln^2\lambda^2\gamma)(\eta\tilde L\gamma)}{(1-\sqrt{\rho})^2(1-\rho)}+\frac{9\eta Ln^2\lambda^2\gamma}{2}\Big)\sum_{t=1}^K\gamma_t\bE[\|\nabla f(\bar\vx_t)\|^2]+\frac{9L\sigma_s^2}{2\eta}\sum_{t=1}^K\gamma_t^2\\
&\quad+\frac{48\rho^2(1+\sqrt{\rho})^2nL\tilde L\sigma_s^2}{(1-\rho)^3}\sum_{t=1}^K\gamma_t^3+\frac{9Ln\gamma(f(\bar\vx_1)-f^\star)}{2}+\frac{6nL\tilde L\tilde C_0}{1+\rho}\bigg)\\
&\leq 6Ln\gamma(f(\bar\vx_1)-f^\star)+8nL\tilde L\tilde C_0\\
&\quad+\frac{6L\sigma_s^2}{\eta}\sum_{t=1}^K\gamma_t^2+\frac{64\rho^2(1+\sqrt{\rho})^2nL\tilde L\sigma_s^2}{(1-\rho)^3}\sum_{t=1}^K\gamma_t^3+\frac{1}{3}\sum_{t=1}^K\gamma_t\bE[\|\nabla f(\bar\vx_t)\|^2]
\ena
where the last inequality follows from that 
\bee
\frac{72\rho^2(\eta Ln^2\lambda^2\gamma)(\eta\tilde L\gamma)}{(1-\sqrt{\rho})^2(1-\rho)}+\frac{9\eta Ln^2\lambda^2\gamma}{2}\leq\frac{72\rho^2(1-\rho)(1+\sqrt{\rho})^2+5(1+\rho)^2(1-\rho)^2}{24(1+\rho)^4}<\frac{1}{4}
\ene
where in the first inequality we used that $\eta Ln^2\lambda^2\gamma\leq\frac{(1-\rho)^2}{24(1+\rho)^2}$ and $\eta\tilde L \gamma\leq\frac{(1-\rho)^2}{(1+\rho)^2}$; the last inequality can be readily verified since it is equivalent to checking the positiveness of a polynomial on $\rho\in(0,1)$, and we omit the details.

Combine \cref{eq14_s5}, \cref{eq_psi}, and \cref{eq_pstar}, and subtract $\frac{1}{3}\sum_{t=1}^K\gamma_t\bE[\|\nabla f(\bar\vx_t)\|^2]$ from both sides. We obtain that
\bea
&\sum_{t=1}^{K}\Big(\gamma_t\bE[\|\nabla f(\bar\vx_{t})\|^2]+nL^2\gamma_t\bE[\|X_t-\bone\bar\vx_{t}^\sT\|_\sF^2]\Big)\\
&\leq \frac{12n(f(\bar\vx_{1})-f^\star)}{\eta}+\frac{9L\sigma_s^2}{\eta}\sum_{t=1}^K\gamma_t^2+\frac{96\rho^2(1+\sqrt{\rho})^2nL\tilde L\sigma_s^2}{(1-\rho)^3}\sum_{t=1}^K\gamma_t^3+12nL\tilde L\tilde C_0.
\ena
The desired result  follows by noticing that
\bee\label{eq_misc2}
\tilde C_0\leq\frac{2\sqrt{\rho}\gamma_1}{1-\sqrt{\rho}}\|X_1-\bone\bar\vx_1^\T\|_\sF^2+\frac{2\rho^2(1+\rho)\gamma_1^3}{(1-\sqrt{\rho})^3}\bE[\|Y_1-\bone\bar\vy_1^\T\|_\sF^2].
\ene

\subsection{Proof of \cref{theo2}}

Let $\tilde f_k\triangleq\sum_{i=1}^{n}\left[f_i(\vx_{i,k})+\nabla f_i(\vx_{i,k})^\T(\bar\vx_k-\vx_{i,k})\right]$, we have $\tilde f_k\leq f(\bar\vx_k)$ and
\bea\label{eq12_s4}
f(\vx)=\sum_{i=1}^{n}f_i(\vx)&\geq\sum_{i=1}^{n}\left[f_i(\vx_{i,k})+\nabla f_i(\vx_{i,k})^\T(\vx-\vx_{i,k})\right]=\tilde{f_k}+n\vg_{k}^\T(\vx-\bar\vx_{k})
\ena
due to the convexity. Let $e_k=\|\bar\vx_k-\vx^\star\|^2$. We have
\bea
e_{k+1}=\|\bar\vx_{k}-\vx^\star-\gamma_k\bar\vy_k\|^2=e_{k}-2\gamma_k\bar\vy_k^\T(\bar\vx_{k}-\vx^\star)+\gamma_k^2\|\bar\vy_k\|^2.
\ena
Since $\bE[\bar\vy_k^\T(\bar\vx_k-\vx^\star)|\cF_{k-1}]=\eta\bar\vg_k^\T(\bar\vx_k-\vx^\star)$, we have
\bea\label{eq3_s5}
\bE[e_{k+1}]&=\bE[e_{k}]-2\gamma_k \eta\bE[\vg_k^\T(\bar\vx_{k}-\vx^\star)]+\gamma_k^2\bE[\|\bar\vy_k\|^2]\\
&\leq \bE[e_{k}]+\frac{2\gamma_k\eta}{n}(f^\star-\bE[\tilde{f_k}])+\gamma_k^2\bE[\|\bar\vy_k\|^2]
\ena
where the inequality is from \cref{eq12_s4}. Then, we have
\bea
&\frac{2\eta}{n}\sum_{t=1}^k\gamma_t\Big(\bE[f(\bar\vx_{t})]-f^\star+\frac{L}{2}\|X_t-\bone\bar\vx_t^\T\|_\sF^2\Big)\\
&\leq\frac{2\eta}{n}\sum_{t=1}^k\gamma_t(\bE[\tilde f_t]-f^\star)+\frac{2\eta  L }{n}\sum_{t=1}^{k}\gamma_t\bE[\|X_t-\bone\bar\vx_t^\T\|_\sF^2]\\
&\leq \bE[e_{1}]-\bE[e_{k+1}]+\frac{\eta}{Ln^2}\Big(\frac{Ln^2}{\eta}\sum_{t=1}^{k}\gamma_t^2\bE[\|\bar\vy_{t}\|^2]+2nL^2\sum_{t=1}^{k}\gamma_t\bE[\|X_t-\bone\bar\vx_t^\T\|_\sF^2]\Big)\\
\ena
where the first inequality follows from $f(\bar\vx_k)\leq\tilde f_k+\frac{L}{2}\|X_k-\bone\bar\vx_k^\T\|_\sF^2$ due to the Lipschitz smoothness, and the last inequality follows from \cref{eq3_s5}. Note that the last term is identical to  the $\psi$ defined in \cref{eq14_s5}. In view of \cref{eq_psi} and \cref{eq_pstar}, we obtain 
\bea
&\sum_{t=1}^k\gamma_t\Big(\bE[f(\bar\vx_{t})]-f^\star+\frac{L}{2}\|X_t-\bone\bar\vx_t^\T\|_\sF^2\Big)\\
&\leq\frac{n(\bE[e_{1}]-\bE[e_{k+1}])}{2\eta}+3\gamma(f(\bar\vx_1)-f^\star)+4\tilde L\tilde C_0\\
&\quad+\frac{3\sigma_s^2}{\eta n}\sum_{t=1}^K\gamma_t^2+\frac{32\rho^2(1+\sqrt{\rho})^2\tilde L\sigma_s^2}{(1-\rho)^3}\sum_{t=1}^K\gamma_t^3+\frac{1}{3}\sum_{t=1}^K\gamma_t(\bE[f(\bar\vx_{t})]-f^\star)\\
&\leq\frac{3ne_1}{\eta}+\frac{5\sigma_s^2}{\eta n}\sum_{t=1}^K\gamma_t^2+\frac{48\rho^2(1+\sqrt{\rho})^2\tilde L\sigma_s^2}{(1-\rho)^3}\sum_{t=1}^K\gamma_t^3+6\tilde L\tilde C_0
\ena
where we used $\|\nabla f(\vx)\|^2\leq 2nL(f(\vx)-f^\star)$ from the convexity and Lipschitz smoothness of $f$ to obtain the first inequality; the last inequality follows by subtracting $\frac{1}{3}\sum_{t=1}^K\gamma_t(\bE[f(\bar\vx_{t})]-f^\star)$ from both sides and then multiply both sides with $3/2$. We also used $f(\vx)-f^\star\leq \frac{nL}{2}\|\bar\vx-\vx^\star\|^2$. The desired result then follows in view of \cref{eq_misc2}.

\section{Discussion}\label{sec6}

In this section, we highlight several key observations from the theoretical results. We illustrate that the network topology and the weight matrix may not affect the convergence rate too much in some cases, and we discuss the relation between DSGT and the centralized SGD. Then, we show that a linear speedup is achievable under some assumptions that are implicitly assumed in existing works.

\subsection{Network independent convergence rate for small $\gamma_k$}
\cref{theo1} implies that the convergence rate of DSGT is independent of the algebraic connectivity $\rho$ for sufficiently small stepsize $\gamma_t$. Specifically, if the following condition holds:
\begin{assum}\label{A1}
$\{\gamma_k\}$ satisfies that $\frac{n\tilde L\rho^2}{(1-\rho)^3}\sum_{t=1}^k\gamma_t^3=O(\sum_{t=1}^k\gamma_t^2)$.
\end{assum}
Then, it follows from \cref{theo1} that the convergence rate of DSGT is
\bee\label{eq_sec4}
R(k)=O\Big(\frac{1+L\sigma_s^2\sum_{t=1}^k\gamma_t^2}{\sum_{t=1}^k\gamma_t}\Big)
\ene
where the constant factor is independent of $\rho$. We borrow the name from  \cite{pu2020asymptotic} to call this property network independence\footnote{\cite{pu2020asymptotic} mainly studies the asymptotic network independence property for D-SGD and focuses on the comparison with the centralized SGD.}. In the case of constant stepsize (\cref{coro2}), \cref{A1} holds when the number of iterations $K$ is large enough. For diminishing stepsizes (\cref{coro1}), this condition trivially holds if $\gamma_k=O(1/\sqrt{k})$.

The network independence property has also been observed for D-PSGD \cite{pu2020asymptotic,lian2017can,neglia2020decentralized} under different conditions. Note that this property for DSGT is previously shown only for strongly convex functions and diminishing stepsize $\gamma_k=O(1/k)$ \cite{pu2018distributed2}.

\subsection{Comparison with the centralized SGD}\label{subsection}
Consider the centralized mini-batch SGD with batch size $M>0$ and stepsize $\bar\gamma_k$, which has the following update rule 
\bee\label{csgd}
\vx_{k+1}=\vx_k-\frac{\bar\gamma_k}{M}\sum_{i=1}^M\nabla_x l(\vx_k;d_i).
\ene
where the $M$ samples are randomly sampled from \emph{global} dataset $\cD=\cup_{i=1}^{n}\cD_i$ with $|\cD|=N=\sum_{i=1}^n N_i$. To characterize its convergence rate, we must introduce the Lipschitz smoothness assumption for the \emph{global} objective function and the variance for \emph{global} SGs, i.e., there exist $L_c\geq0$ and $\sigma\geq0$ such that
\begin{equation}\label{assum_cen}
\begin{gathered}
\|\nabla f(\vx)-\nabla f(\vy)\|\leq L_c\|\vx-\vy\|,\forall \vx,\vy,\\
\bE_{d\sim\cD}[\big\|\nabla_x l(\vx;d)-\frac{1}{N}\sum\nolimits_{u=1}^N \nabla_x l(\vx;d_u)\big\|^2]\leq\sigma^2.
\end{gathered}
\end{equation}
Note that \cref{assum1} in this work is defined for \emph{local} objective functions and SGs. From \cite{lan2020first}, the convergence rate of \cref{csgd} is\footnote{The rate can be deduced from Theorem 6.1 in \cite{lan2020first}, in which the global cost function $\bar f$ is defined as $f$ in \cref{obj2} divided by the number of samples, i.e., $\bar f=\frac{1}{N}f$.} 
\bee\label{csgd2}
\frac{\sum_{t=1}^k \bar\gamma_t\bE[\|\nabla f(\vx_t)\|^2]}{N\sum_{t=1}^k \bar\gamma_t}\leq O\Big(\frac{f(\vx_1)-f^\star+L_c\sigma^2/M\sum_{t=1}^k\bar\gamma_t^2}{\sum_{t=1}^k \bar\gamma_t}\Big).
\ene

By \cref{theo1}, the convergence rate of DSGT when $\lambda=0$ is
\bee\label{rate_dsgt}
\frac{\sum_{t=1}^k \bar\gamma_t\bE\big[\|\nabla f(\vx_t)\|^2+nL^2\|X_t-\bone\bar\vx_{t}^\sT\|_\sF^2\big]}{N\sum_{t=1}^k \bar\gamma_t}\leq O\Big(\frac{f(\vx_1)-f^\star+nL\sigma_s^2/M^2\sum_{t=1}^k\bar\gamma_t^2}{\sum_{t=1}^k \bar\gamma_t}\Big)
\ene
where we used \cref{A1}, and set $\eta=\frac{M}{N}$ and stepsize $\gamma_k=n\bar\gamma_k/M$.

In order to compare \cref{csgd2} with \cref{rate_dsgt}, we must study the relations between $L_c,\sigma^2$ and $L,\sigma_s^2$. To this end, we make the following assumption:

\begin{assum}\label{A2}
There exists a $c_l>0$ such that $L_i=c_l N_i$ and $L_c=c_l N$, and $N_i=N_j,\forall i,j\in\cV$.
\end{assum}
\cref{A2} states that the Lipschitz constant of local objective function $f_i$ is proportional to the size of local datasets, which is reasonable since $f_i$ is the sum of $N_i$ functions. It generally holds for convex loss function $l(\vx;d)$. We also assume that local datasets have equal sizes, and hence $L_c=nL=nL_i,\forall i\in\cV$. We consider two scenarios:
\begin{enumerate}[leftmargin=18pt, labelsep=4pt, label=(\alph*)]
	\item Given $n$ local datasets and \cref{assum1}, what is the convergence rate if one applies the centralized SGD \cref{csgd} to the global dataset?
	
	$\quad$In this case, we have $L_c= nL$ under \cref{A2} and $\sigma^2\geq \sigma_s^2/M$ (see \cref{adx_6} for a proof). Substituting it into \cref{csgd2} and then comparing it with \cref{rate_dsgt} leads to an interesting observation --- DSGT can be faster than the centralized SGD. We use a toy example to illustrate it. 
	\begin{example}\label{exm2}
		Consider that each local dataset consists of two duplicated samples and different local datasets have different samples. Let the size of mini-batch be one.  In DSGT, a local SG is just the half of the local \emph{full} gradient, and hence DSGT has a convergence rate $O(1/k)$ (by setting $\sigma_s=0$ in \cref{coro2} or referring to \cite{qu2017harnessing}). In the centralized SGD, the SGs have positive variance ($\sigma>0$) since the global dataset has distinct samples, and hence the convergence rate is at most $O(1/\sqrt{k})$, which is slower than that of DSGT. Interestingly, we do not need \cref{A2} here.
	\end{example}
	
	\item Given a global dataset satisfying \cref{assum_cen}. If we partition it into $n$ local datasets, what is the convergence rate if one applies DSGT on these local datasets?
	
	$\quad$In this case, we need to bound $\sigma_s^2$ with $\sigma^2$, and select $\sigma_s^2\leq\min\{nM\sigma^2,2M\sigma^2+2\eta Lc_l\|X_k-\bone\bar\vx_k^\T\|^2/M\}$ in the $k$-th iteration (see \cref{adx_6} for a proof). Then,
	\bee\label{eq1_sec4}
	\frac{nL\bE[\sigma_s^2]}{M^2}\sum_{t=1}^k\bar\gamma_t^2\leq \frac{2L_c\sigma^2}{M}\sum_{t=1}^k\bar\gamma_t^2+\frac{2n L^2}{MN}\sum_{t=1}^k\bar\gamma_t\bE[\|X_t-\bone\bar\vx_t^\T\|^2
	\ene
	where we used $\bar\gamma_t\leq\frac{N}{L_c}=\frac{1}{c_l}$. Plug it into \cref{rate_dsgt} and notice that the second term in the right-hand-side of \cref{eq1_sec4} can be eliminated by subtracting it from both sides of \cref{rate_dsgt}. Then, we will obtain a rate that is identical to the rate \cref{csgd2} of the centralized SGD. 
\end{enumerate}

In both scenarios, DSGT is comparable or even better than the centralized SGD. Nevertheless, we emphasize that the above analysis relies  on \cref{A1} and \cref{A2}.

\subsection{Speedup analysis}\label{sec3.2}

Speedup is defined as $S=T_1/T_n$, where $T_n$ is the time cost to solve \cref{obj2} within a given degree of accuracy using $n$ workers with equal computational capabilities, which corresponds to the second scenario in the last subsection. Speedup is important to reflect the scalability of a decentralized algorithm \cite{lian2017can,hao2019linear} and is apparently bounded by $n$. Generally speaking, 
$T_n$ mainly consists of computation time, communication time, and synchronization time. We assume in this section that communication time and synchronization time are negligible since they are difficult to quantified in the current framework and vary with applications and physical devices. In some applications communications and synchronization can take much time, where the speedup may be largely reduced. One may consider to combine DSGT with communication compressed techniques \cite{lan2017communication,shen2018towards,koloskova2019decentralized} or asynchronous methods \cite{lian2018asynchronous,zhang2019asyspa,zhang2019asynchronous} in future works.

Let $\bar T_n$ be the time for $n$ workers to finish one iteration in DSGT, and $K_n$ be the number of iterations to obtain an $\epsilon$-optimal solution with $n$ workers, i.e., $\frac{1}{K_n}\sum_{t=1}^{K_n}\bE[\|\nabla f(\bar\vx_t)\|^2]\leq \epsilon$. Each node processes $\eta N_i$ samples at each iteration, and thus node $i$ needs $O(N_i)$ time to finish an iteration. Hence, $\bar T_n=O(\max_{i\in\cV} N_i)$ as DSGT is a synchronous algorithm. Moreover, \cref{coro2} implies that  $K_n=O(\frac{nL\sigma_s^2}{\epsilon^2}+\frac{n^2L\rho^2}{(1-\rho)^3\epsilon})$ if a constant stepsize is adopted and $\lambda=0$, where the constant factor is independent of $n,\rho,L,\sigma_s$. If \cref{A1} holds, the first term of $K_n$ will dominate. Then, the speedup can be represented in the following form
\bee
S= \frac{\bar T_1}{\bar T_n}\times\frac{K_1}{K_n}=\frac{\sum_{i=1}^n N_i}{\max_{i\in\cV}N_i}\times\frac{O(L_c)}{O(nL)}.
\ene

If \cref{A2} holds, we have $\sum_{i=1}^n N_i=n\max_{i\in\cV}N_i$ and  $L_c=nL$, and thus $S=O(n)$. A linear speedup is then achievable. It is worth noting that many existing works (e.g. D-PSGD\cite{lian2017can}, D$^2$\cite{tang2018d}, SGP\cite{assran2018stochastic}, \cite{hao2019linear}) implicitly assumed \cref{A1} and \cref{A2} to derive linear speedup properties.

If either \cref{A1} or \cref{A2} is not satisfied, then the speedup of DSGT may become sublinear, which might holds for any decentralized algorithm. We provide some intuitive explanations. Consider an example where $f_1(\vx)$ is a non-constant function and $f_i(\vx)=0$ for $i=2,\ldots,n$. We have $f(\vx)=\sum_{i=1}^{n}f_i(\vx)=f_1(\vx)$. Clearly, \cref{A2} does not hold. Since nodes $i$ ($i=2,\ldots,n$) compute local (stochastic) gradients, their existence do not help to minimize $f$. Thus, it seems impossible for any decentralized algorithm to have a linear speedup. Furthermore, if local datasets have distinct sizes, then nodes with smaller datasets have to wait for those with larger datasets at each iteration, which introduces idle time and may lead to a smaller $S$. \cref{A1} essentially states that the consensus speed is not the bottleneck of the global convergence rate. Otherwise, the design of weight matrix and network topology will significantly affect $S$.

Finally, it is worth noting that \cref{A1} is easy to satisfy for a smaller $\rho$, which corresponds to a more densely connected graph. We refer interested readers to \cite{pu2018distributed2} for the values of $\rho$ for several common graphs. However, a dense graph will introduce heavy communication overhead, which may slow down the convergence rate in wall-clock time. Thus, there is a trade-off in the design of communication topology. In practice, it is preferable to design a network with a given $\rho$ by as few edges as possible to reduce communication overhead, which is addressed in \cite{xiao2004fast}. To our best knowledge, most of existing decentralized methods (e.g. \cite{lian2017can,tang2018d,assran2018stochastic}) face such a trade-off and it is empirically observed in \cite{assran2018stochastic}.

\section{Experiment}\label{sec4}

We numerically evaluate the DSGT on two tasks and compare it with D-PSGD \cite{lian2017can} and D$^2$ \cite{tang2018d}. Note that SGP \cite{assran2018stochastic} reduces to D-PSGD for a symmetric weight matrix in undirected graphs. The first task is to train a logistic regression classifier on the CIFAR-10 dataset \cite{krizhevsky2009learning}, which is a convex problem. The second task is to train LeNet \cite{lecun1998gradient} on the CIFAR-10, which is a non-convex neural network training problem. 

We test over $n=1, 6, 12, 18, 24$ nodes\footnote{Each node is a physical core of Intel Xeon CPU E5-2660 v4 on a server.}, and each node is assigned a randomly partitioned subset with equal sizes. The sum of mini-batch sizes among nodes is 1024 and hence each node uses a batch size $1024/n$, which is for a fair comparison with the centralized SGD. We use a constant stepsize tuned for each algorithm and each task from the grid $[0.001n, 0.005n, 0.01n, 0.03n, 0.05n, 0.1n, 0.2n]$. The communication graph is constructed using FDLA\cite{xiao2004fast}. In particular, we first randomly generate a topology where each node has $2\log_2{n}$ neighbors on average, and determine the weight matrix $W$ by applying the FDLA method \cite{xiao2004fast} for DSGT and D-PSGD, and the Metropolis method for D$^2$. The weight matrix obtained by the FDLA method generally has smaller $\rho$ than Metropolis method, although it may result in the divergence of D$^2$ as explained in Section \ref{sec2.2}. The implementation is based on PyTorch 1.1 and OpenMPI 1.10.

\begin{figure}[!t]
	\vspace{-5pt}
	\centering
	\subfloat[Loss vs epochs over 12 nodes]{\label{fig_lra}\includegraphics[width=0.249\linewidth]{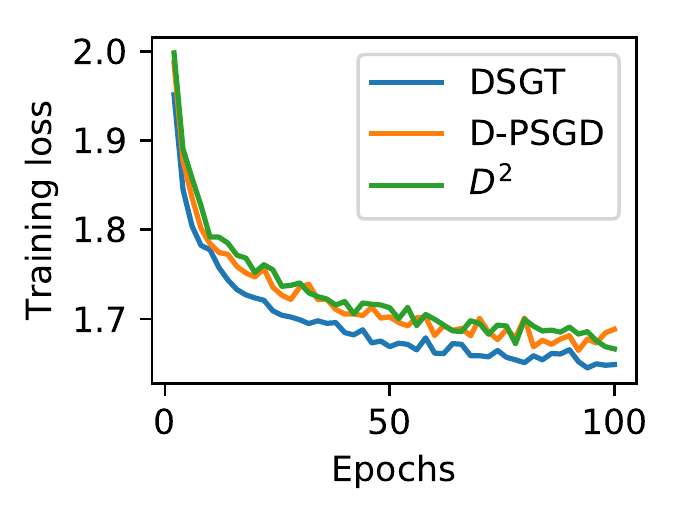}}
	\subfloat[Loss vs time over 12 nodes]{\label{fig_lrb}\includegraphics[width=0.249\linewidth]{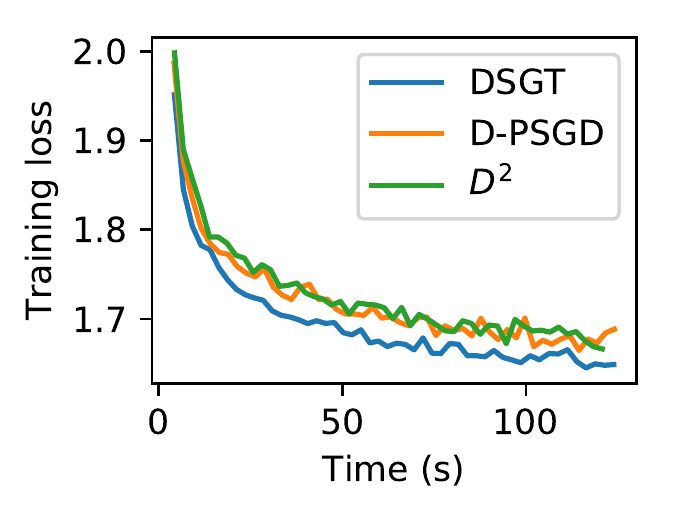}}
	\subfloat[Loss of DSGT]{\label{fig_lrc}\includegraphics[width=0.249\linewidth]{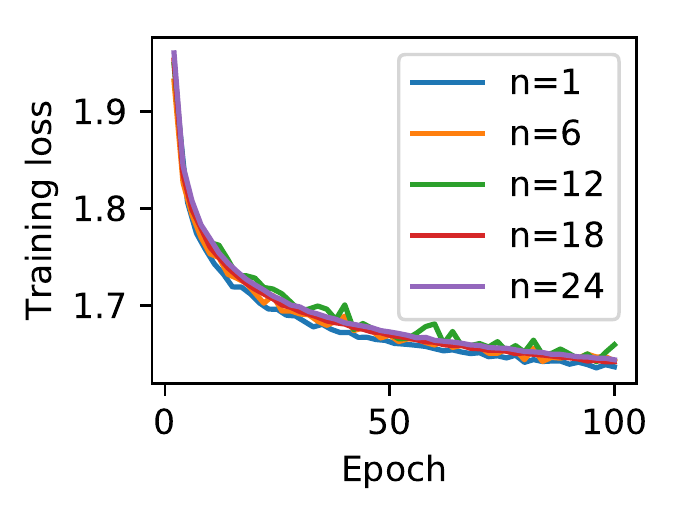}}
	\subfloat[Speedup in training time]{\label{fig_lrd}\includegraphics[width=0.249\linewidth]{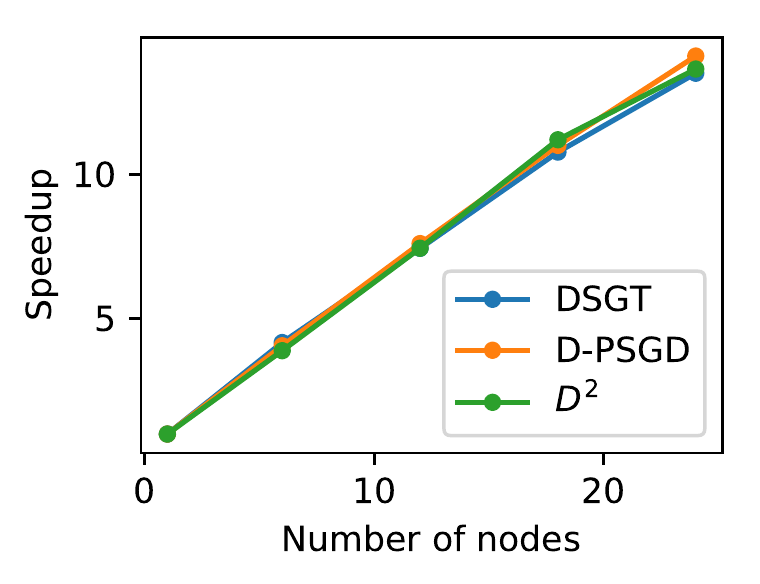}}
	\caption{Results on logistic regression task.}
	\label{fig_lr}
	\vspace{-5pt}
\end{figure}

\begin{figure}[!t]
	\vspace{-5pt}
	\centering
	\subfloat[Loss vs epochs over 12 nodes]{\label{fig_lenet1}\includegraphics[width=0.249\linewidth]{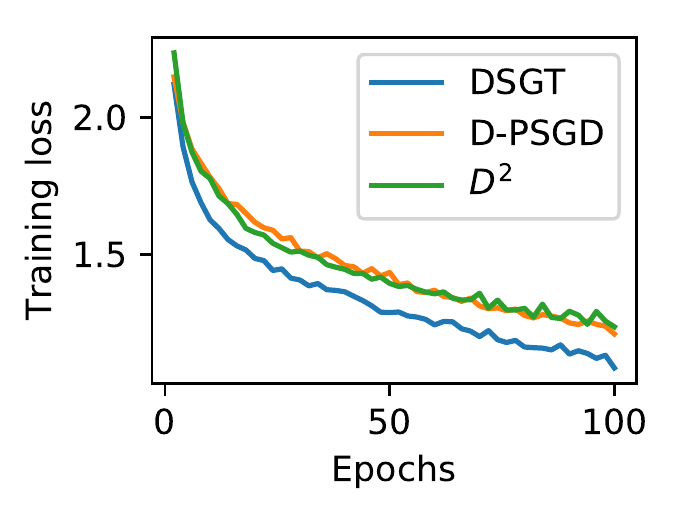}}
	\subfloat[Loss vs time over 12 nodes]{\label{fig_lenet2}\includegraphics[width=0.249\linewidth]{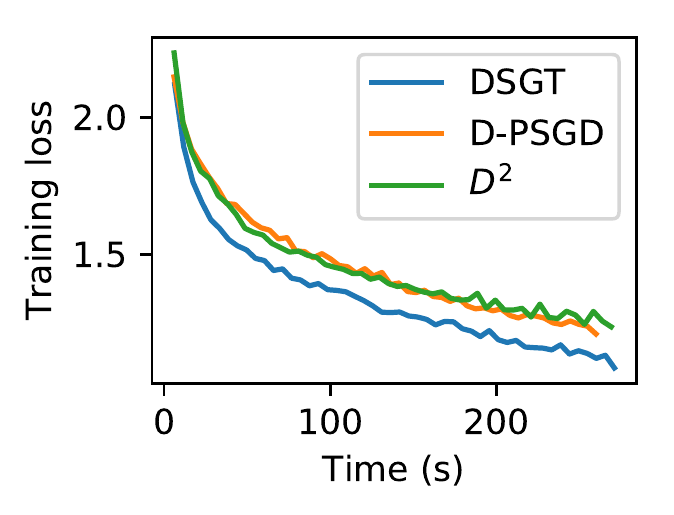}}
	\subfloat[Loss of DSGT]{\label{fig_lenet3}\includegraphics[width=0.249\linewidth]{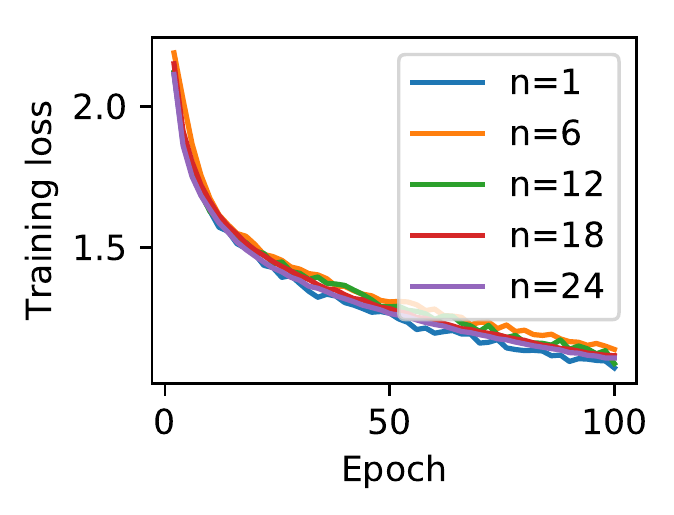}}
	\subfloat[Speedup in training time]{\label{fig_lenet4}\includegraphics[width=0.249\linewidth]{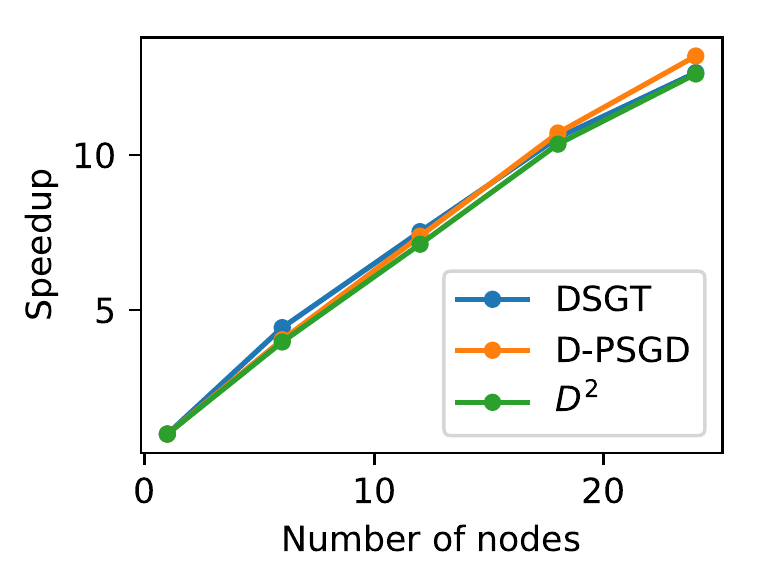}}
	\caption{Results on LeNet.}
	\label{fig_lenet}
	\vspace{-5pt}
\end{figure}

For the logistic regression task, Figs. \subref*{fig_lra} and \subref*{fig_lrb} depict the decreasing training loss w.r.t. number of epochs and the wall-clock time over 12 nodes, respectively. We can observe that DSGT has a faster convergence rate than both D-PSGD and D$^2$. The cases for other numbers of nodes are almost identical, and hence are omitted for saving space. Fig. \subref*{fig_lrc} plots the training loss of DSGT w.r.t. epochs over different number of nodes, which shows that the number of iterations required to achieve a certain accuracy is essentially not related to the number of nodes, which validates our results in Section \ref{sec3.2}. Fig. \subref*{fig_lrd} illustrates the linear speedup of DSGT in training time, which is also consistent with Section \ref{sec3.2}. The performance of DSGT on LeNet is similar as displayed in Fig. \ref{fig_lenet}.

We also compute the classification accuracy, which is the ratio of number of correct predictions to the total number of input samples, of the trained models on the test data and the results are presented in Fig. \ref{fig_acc}. It shows that DSGT has a higher accuracy than D-PSGD and D$^2$, which empirically validates the superiority of DSGT from a different point of view. The performance is similar for other numbers of nodes.

\begin{figure}[!t]
	\vspace{-5pt}
	\centering
	\subfloat[LeNet]{\label{fig_acca}\includegraphics[width=0.29\linewidth]{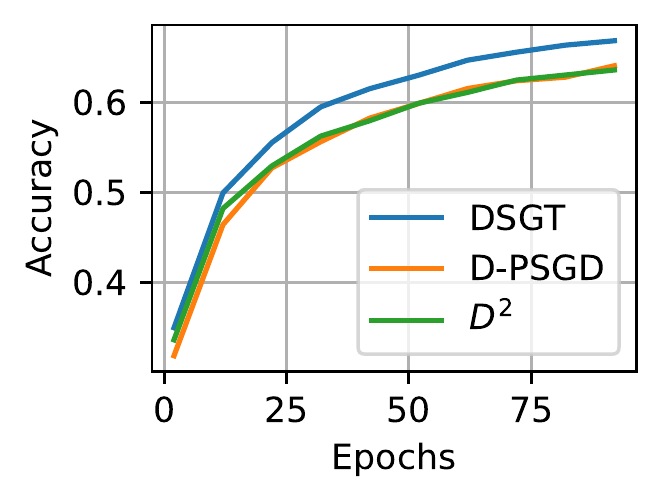}}\qquad
	\subfloat[Logistic regression]{\label{fig_accb}\includegraphics[width=0.3\linewidth]{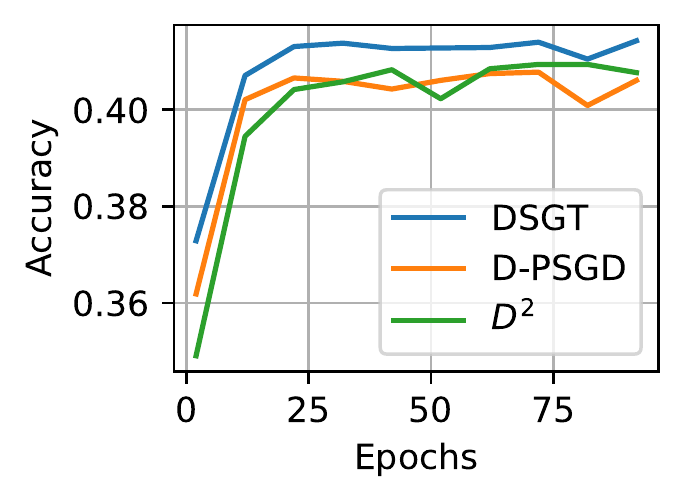}}
	\caption{Classification accuracy on the test data v.s. training epochs over 12 nodes.}
	\label{fig_acc}
	\vspace{-5pt}
\end{figure}

\section{Conclusion}\label{sec5}
We studied DSGT for non-convex empirical risk minimization problems. DSGT leverages the SG tracking method to handle decentralized datasets with different variances and sizes. We have proved the convergence rate of DSGT for non-convex functions with both constant and diminishing stepsizes. We show that DSGT has a so-called network independence property under certain conditions, and hence its convergence rate is comparable with the centralized SGD w.r.t. the number of iterations. Since DSGT with $n$ nodes can finish one iteration $n$ time faster than the centralized SGD (under some conditions), it can achieve a linear speedup in computation time w.r.t. the number of nodes under the same assumptions as existing works.  Experiments validate the efficiency of DSGT and show its advantages over existing algorithms. Future works shall focus on accelerating DSGT with momentum or variance reduction methods, and reduce communication and synchronization overhead with communication-efficient methods and asynchronous updates.

\section*{Acknowledgments}
The authors would like to thank Dr. Shi Pu from Chinese University of Hong Kong (Shenzhen) for his helpful suggestions on showing the network independence property of DSGT. We also appreciate the constructive comments from Associate Editor and anonymous reviewers, which greatly improved the quality of this work.

\appendix

\section{Proof of \cref{lemma1}}\label{adx0}

We define $J=\frac{1}{n}\bone\bone^\T$ for ease of presentation. Note that $\|W-J\|=\rho<1,\ \|I-J\|=1$
and $\|AB\|_\sF\leq\|A\|\|B\|_\sF,\forall A,B$, which will be frequently used in the subsequent proofs. To obtain \cref{eq1_s2}, it follows from \cref{alg} that
\bea
&\|X_{k+1}-\bone\bar\vx_{k+1}^\T\|_\sF^2=\|WX_k-\bone\bar\vx_{k}^\T+\gamma_k\bone\bar\vy_{k}^\T-\gamma_k WY_k\|_\sF^2\\
&\leq\|WX_k-\bone\bar\vx_k^\T\|_\sF^2-2\gamma_k(WX_k-\bone\bar\vx_k^\T)^\T(WY_k-\bone\vy_k^\T)+\gamma_k^2\|WY_k-\bone\vy_k^\T\|_\sF^2\\
\ena
and the second term in the right-hand-side can be bounded by
\bea
&-2(WX_k-\bone\bar\vx_k^\T)^\T(WY_k-\bone\vy_k^\T)\\
&\quad\leq\frac{1-\rho^2}{(1+\rho^2)\gamma_k}\rho^2\|X_k-\bone\bar\vx_k^\T\|_\sF^2+\frac{(1+\rho^2)\gamma_k}{1-\rho^2}\rho^2\|Y_k-\bone\vy_k^\T\|_\sF^2
\ena
where  we used $2\va^\T\vb\leq c\|\va\|^2+\frac{1}{c}\|\vb\|^2$ for any $c>0$ and $\|WX_k-\bone\bar\vx_k^\T\|_\sF=\|W-J)(X_k-\bone\bar\vx_k^\T)\|_\sF\leq\rho\|X_k-\bone\bar\vx_k^\T\|$. Combing the above relations implies that
\bea
\|X_{k+1}-\bone\bar\vx_{k+1}^\T\|_\sF^2\leq \frac{2\rho^2}{1+\rho^2}\|X_k-\bone\bar\vx_k^\T\|_\sF^2+\frac{2\rho^2\gamma_k^2}{1-\rho^2}\|Y_k-\bone\vy_k^\T\|_\sF^2.
\ena
Taking expectations on both sides implies the result.

We now turn to \cref{eq1_s1}. To this end, we let $\partial_k\triangleq\partial F(X_{k};\bxi_k)$ and $\nabla_k\triangleq\nabla F(X_k)$ for ease of presentation. It follows from \cref{alg} that
\bea\label{eq2_s1}
&\bE[\|Y_{k+1}-\bone\bar\vy_{k+1}^\T\|_\sF^2]=\bE[\|(W-J)Y_k+(I-J)(\partial_{k+1}-\partial_k)\|_\sF^2]\\
&=\bE[\|(W-J)Y_k+(I-J)(\partial_{k+1}-\eta\nabla_{k+1}-(\partial_k-\eta\nabla_k)+(\eta\nabla_{k+1}-\eta\nabla_k))\|_\sF^2]\\
&\leq\bE[\|WY_k-\bone\bar\vy_{k}^\T\|_\sF^2]+\underbrace{\bE[\|\partial_{k+1}-\eta\nabla_{k+1}-(\partial_k-\eta\nabla_k)\|_\sF^2]}_{(a)}+\eta^2\underbrace{\bE[\|\nabla_{k+1}-\nabla_k\|_\sF^2]}_{(b)}\\
&\quad+2\underbrace{\bE[Y_k^\T(W-J)(\partial_{k+1}-\eta\nabla_{k+1}-(\partial_k-\eta\nabla_k))]}_{(c)}+\underbrace{\bE[2Y_k^\T(W-J)(\eta\nabla_{k+1}-\eta\nabla_k)]}_{(d)}\\
&\quad+2\underbrace{\bE[(\partial_{k+1}-\eta\nabla_{k+1}-(\partial_k-\eta\nabla_k))(I-J)(\eta\nabla_{k+1}-\eta\nabla_k)]}_{(e)}
\ena

The rest of the proof is devoted to bound $(a),(b),(c),(d),(e)$. Let us first consider $(a)$. Let $\Sigma=[\sigma_1,\dots,\sigma_n]^\T\in\bR^{n}$ and $\nabla \vf(X_k)=[\nabla f(\vx_{1,k}),\dots,\nabla f(\vx_{n,k})]^\T\in\bR^{n\times m}$. We have that $\nabla f$ and hence $\nabla \vf$  are $(nL)$-Lipschitz smooth. Therefore,
\bea
\|\nabla \vf(X_{k})\|_\sF^2&\leq2\|\nabla \vf(X_{k})-\nabla \vf(\bone\bar\vx_k^\T)\|_\sF^2+2\|\nabla \vf(\bone\bar\vx_k^\T)\|_\sF^2\\
&\leq 2n^2L^2\|X_k-\bone\bar\vx_k^\T\|_\sF^2+2n\|\nabla f(\bar\vx_k)\|^2
\ena
and
\bea
&\|\nabla \vf(X_{k+1})\|_\sF^2=2\|\nabla \vf(WX_k-\gamma_k WY_k)-\nabla \vf(WX_k)\|_\sF^2+2\|\nabla \vf(WX_k)\|_\sF^2\\
&\leq 2\gamma_k^2 n^2L^2\|WY_k\|_\sF^2+4\|\nabla \vf(WX_k)-\nabla \vf(\bone\bar\vx_k)\|_\sF^2+4\|\nabla \vf(\bone\bar\vx_k)\|_\sF^2\\
&\leq 2\gamma_k^2 n^2L^2(\|WY_k-\bone\bar \vy_k^\T\|_\sF^2+n\|\bar\vy_k\|^2)+4n^2L^2\|WX_k-\bone\bar\vx_k^\T\|_\sF^2+4n\|\nabla f(\bar\vx_k)\|^2\\
&\leq 2(\gamma_k \rho nL)^2\|Y_k-\bone\bar \vy_k^\T\|_\sF^2+4(\rho nL)^2\|X_k-\bone\bar\vx_k^\T\|_\sF^2+4n\|\nabla f(\bar\vx_k)\|^2+2\gamma_k^2L^2n^3\|\bar\vy_k\|^2.
\ena
Then, it follows from \cref{assum1}(c) that 
\bea\label{eq6_s1}
(a)&=\bE[\|(\partial_{k+1}-\eta\nabla_{k+1})-(\partial_k-\eta\nabla_k)\|_\sF^2]=\bE[\|\partial_{k+1}-\eta\nabla_{k+1}\|_\sF^2]+\bE[\|\partial_k-\eta\nabla_k\|_\sF^2]\\
&\leq 2\|\Sigma\|^2+(\lambda\eta)^2\bE[\|\nabla \vf(X_k)\|_\sF^2]+(\lambda\eta)^2\bE[\|\nabla \vf(X_{k+1})\|_\sF^2]\\
&\leq 2\sigma_s^2+2(\lambda\eta nL)^2(1+\rho)^2\bE[\|X_k-\bone\bar\vx_k^\T\|_\sF^2]+2(\gamma_k\rho\lambda\eta nL)^2\bE[\|Y_k-\bone\bar \vy_k^\T\|_\sF^2]\\
&\quad+6(\lambda\eta)^2n\bE[\|\nabla f(\bar\vx_k)\|^2]+2(\gamma_k\lambda\eta L)^2n^3\bE[\|\bar\vy_k\|^2].
\ena
where we used that $\bE[(\partial_{k+1}-\eta\nabla_{k+1})^\T(\partial_k-\eta\nabla_k)]=\bE_{\cF_{k}}[\bE[(\partial_{k+1}-\eta\nabla_{k+1})^\T(\partial_k-\eta\nabla_k)|\cF_{k}]]=\bE_{\cF_{k}}[\bE[(\partial_{k+1}-\eta\nabla_{k+1})|\cF_{k}]^\T(\partial_k-\eta\nabla_k)]=0$ due to the independence between $\bxi_{k}$ and $\bxi_{k+1}$. 

We then turn to $(b)$. It holds that
\bea\label{eq5_s1}
(b)&=\bE[\|\nabla_{k+1}-\nabla_k\|_\sF^2]\leq L^2\bE[\|X_{k+1}-X_k\|_\sF^2]\\
&=L^2\bE[\|WX_{k}-X_k-\gamma_k WY_k\|_\sF^2]=L^2\bE[\|(W-I)(X_k-\bone\bar\vx_k^\T)-\gamma_k WY_k\|_\sF^2]\\
&=L^2\bE[\|(W-I)(X_k-\bone\bar\vx_k^\T)\|_\sF^2]+\gamma_k^2L^2\bE[\|WY_k\|_\sF^2]\\
&\quad-2\gamma_kL^2\bE[(X_k-\bone\bar\vx_k^\T)^\T(W-I)(W-J)Y_k]\\
&\leq2L^2(\|W-J+J-I\|^2\bE[\|X_k-\bone\bar\vx_k^\T\|_\sF^2]+\gamma_k^2L^2\bE[\|(W-J)Y_k\|_\sF^2]\\
&\quad+\gamma_k^2L^2\bE[\|(W-J)(Y_k-\bone\bar\vy_{k}^\T)+\bone\bar\vy_{k}^\T)\|_\sF^2])\\
&\leq2(1+\rho)^2L^2\bE[\|X_k-\bone\bar\vx_k^\T\|_\sF^2]+2\gamma_k^2\rho^2L^2\bE[\|Y_k-\bone\bar\vy_{k}^\T\|_\sF^2]+\gamma_k^2nL^2\bE[\|\bar\vy_{k}\|^2]
\ena
where we used $(W-I)J=0$ and $(W-J)J=0$.

For $(c)$, notice that the conditional expectation $\bE[Y_k^\T(W-J)(\partial_{k+1}-\eta\nabla_{k+1})|\cF_{k}]=0$. Moreover
\bea
&\bE[Y_k^\T(W-J)(\eta\nabla_k-\partial_k)|\cF_{k-1}]\\
&=\bE[(WY_{k-1}+\partial_k-\partial_{k-1})^\T(W-J)(\eta\nabla_k-\partial_k)|\cF_{k-1}]\\
&=\bE[\partial_k^\T(W-J)(\eta\nabla_k-\partial_k)|\cF_{k-1}]=\bE[(\partial_k-\eta\nabla_k)^\T(W-J)(\eta\nabla_k-\partial_k)|\cF_{k-1}]\leq \rho \sigma_s^2
\ena
Therefore,
\bea
&(c)=\bE[Y_k^\T(W-J)(\partial_{k+1}-\eta\nabla_{k+1}-(\partial_k-\eta\nabla_k))]\\
&=\bE_{\cF_{k}}[\bE[Y_k^\T(W-J)(\partial_{k+1}-\eta\nabla_{k+1})|\cF_{k}]]+\bE_{\cF_{k-1}}[\bE[Y_k^\T(W-J)(\eta\nabla_k-\partial_k)|\cF_{k-1}]]\\
&=\rho\sigma_s^2.
\ena

We can bound $(d)$ as follows:
\bea
(d)&=2\bE[Y_k^\T(W-J)(\eta\nabla_{k+1}-\eta\nabla_k)]\\
&\leq\frac{(1-\rho^2)\rho^2}{1+\rho^2}\bE[\|Y_k-\bone\bar\vy_k^\T\|_\sF^2]+\frac{(1+\rho^2)\eta^2}{1-\rho^2}\|\nabla_{k+1}-\nabla_k\|_\sF^2.
\ena
where the last term is bounded by \cref{eq5_s1}.

Finally, we consider $(e)$. Note that $\bE[(\partial_{k+1}-\eta\nabla_{k+1})(I-J)(\eta\nabla_{k+1}-\eta\nabla_k)|\cF_{k}]=0$ and $\bE[(\partial_k-\eta\nabla_k)(I-J)\eta\nabla_k|\cF_{k-1}]=0$. Moreover,
\bea
&\bE[(\partial_k-\eta\nabla_k)(I-J)\nabla_{k+1}|\cF_{k-1}]\\
&=\bE[(\partial_k-\eta\nabla_k)(I-J)(\nabla F(WX_k-\gamma_k W(WY_{k-1}+\partial_k-\partial_{k-1}))|\cF_{k-1}]\\
&=\bE[(\partial_k-\eta\nabla_k)(I-J)(\nabla F(U_k-\gamma_kW\partial_k)-\nabla F(U_k-\gamma_kW\eta \nabla_k))|\cF_{k-1}]\\
&\quad+\bE[(\partial_k-\eta\nabla_k)(I-J)\nabla F(U_k-\gamma_kW\eta\nabla_k)|\cF_{k-1}]\\
&=\bE[(\partial_k-\eta\nabla_k)(I-J)(\nabla F(U_k-\gamma_kW\partial_k)-\nabla F(U_k-\gamma_kW\eta\nabla_k))|\cF_{k-1}]\\
&\leq\bE[\gamma_kL\|W\|\|\partial_k-\eta\nabla_k\|_\sF^2|\cF_{k-1}]\leq \gamma_kL\sigma_s^2
\ena
where $U_k=WX_k-\gamma_k W(WY_{k-1}-\partial_{k-1})$ and we used $\|W\|\leq\sqrt{\|W\|_1\|W\|_\infty}=1$. Thus,
\bea
(e)&=\bE[(\partial_{k+1}-\eta\nabla_{k+1}-(\partial_k-\eta\nabla_k))(I-J)^2(\eta\nabla_{k+1}-\eta\nabla_k)]\\
&=\bE[\bE[(\partial_{k+1}-\eta\nabla_{k+1})(I-J)^2(\eta\nabla_{k+1}-\eta\nabla_k)|\cF_{k}]]\\
&\quad+\bE[\bE[(\partial_k-\eta\nabla_k)(I-J)^2\eta\nabla_k|\cF_{k-1}]]+\eta\bE[\bE[(\partial_k-\eta\nabla_k)(I-J)^2\nabla_{k+1}|\cF_{k-1}]]\\
&\leq \gamma_k\eta L\sigma_s^2
\ena

Combine the above relations with \cref{eq2_s1} implies the desired result.

\section{Proof of a result in Section \ref{subsection}}\label{adx_6}

Let $r_u^{(i)}=\nabla l(\vx_{i,k};d_u^{(i)})$, $\bar r_i=\nabla l(\bar\vx_{k};d_u^{(i)})$, $\vr=[r_1^{(1)},\cdots,r_{N_1}^{(1)},\cdots,r_1^{(n)},\cdots,r_{N_n}^{(n)}]^\T$ for any $i\in\cV$, and $\bar\vr=[\bar r_1^{(1)},\cdots,\bar r_{N_1}^{(1)},\cdots,\bar r_1^{(n)},\cdots,\bar r_{N_n}^{(n)}]^\T$. Let $J_N=I_N-\frac{1}{N}\bone\bone^\T$, $\tilde J_{N_i}=\diag(0,\cdots,J_{N_i},\cdots,0)$, i.e., $\tilde J_{N_i}\in\bR^{N\times N}$ is a $n$-block-diagonal matrix with $J_{N_i}$ on its $i$-th diagonal block. Let $\tilde J_N=\sum_{i=1}^n J_{N_i}$. We have from \cref{assum_cen} that
\bea
N\bE_{d\sim\cD}\Big[\Big\|\nabla_x l(\vx;d)-\sum_{u=1}^N \frac{\nabla_x l(\vx;d_u)}{N}\Big\|^2\Big]=\sum_{u=1}^{N}\Big\|\bar\vr(u)- \sum_{v=1}^N \frac{\bar\vr(v)}{N}\Big\|^2=\|J_N \bar\vr\|^2\leq \sigma^2 N
\ena
and hence we can set $\sigma^2=\max_{\bar\vx_k}\|J_N \bar\vr\|^2/N$. Moreover, it holds that
\bea
&\bE\Big[\sum_{i=1}^n\|\partial f_i(\vx_{i,k};\xi_{i,k})-\eta\nabla f_i(\vx_{i,k})\|^2\Big]=\sum_{i=1}^n\bE_{\xi_{i,k}\sim\cD_i}\Big[\Big\|\partial f_i(\vx_{i,k};\xi_{i,k})-\eta\nabla f_i(\vx_{i,k})\Big\|^2\Big]\\
&=\sum_{i=1}^n\bE_{\xi_{i,k}\sim\cD_i}\Big[\Big\|\sum_{u=1}^{\eta N_i}\nabla_x l(\vx_{i,k};\xi_{i,k}^{(u)})-\eta\nabla f_i(\vx_{i,k})\Big\|^2\Big]\\
&=\sum_{i=1}^n\bE_{\xi_{i,k}\sim\cD_i}\Big[\Big\|\sum_{u=1}^{\eta N_i}\Big(\nabla_x l(\vx_{i,k};\xi_{i,k}^{(u)})-\frac{1}{N_i}\nabla f_i(\vx_{i,k})\Big)\Big\|^2\Big]\\
&\overset{(a)}=\sum_{i=1}^n\eta N_i\bE_{d\sim\cD_i}\Big[\Big\|\nabla_x l(\vx_{i,k};d)-\frac{1}{N_i}\nabla f_i(\vx_{i,k})\Big\|^2\Big]\\
&=\sum_{i=1}^n\eta \sum_{u=1}^{N_i}\|\nabla_x l(\vx_{i,k};d_u^{(i)})-\sum_{v=1}^{N_i}\frac{\nabla_x l(\vx_{i,k};d_v^{(i)})}{N_i}\|^2=\sum_{i=1}^n\eta \|\tilde J_{N_i}\vr\|^2=\eta\|\tilde J_N\vr\|^2\leq\sigma_s^2.
\ena
where $\xi_{i,k}^{(u)}$ is the $u$-th sample in the mini-batch $\xi_{i,k}$ and $(a)$ used the independence of samples in $\xi_{i,k}$. The last equality follows from that $\sum_{i=1}^n \tilde J_{N_i}^2=\tilde J_N^2$. 

Note that $J_N-\tilde J_N$ is positive semidefinite. If $\rho=0$, we have $\vr=\bar\vr$ and we can set $\sigma_s^2=\max_{X_k}\eta\|\tilde J_N \vr\|^2$. Hence, $\sigma^2\geq\sigma_s^2/(\eta N)=\sigma_s^2/M$. Otherwise, we have
\bea
&\bE\Big[\sum_{i=1}^n\|\partial f_i(\vx_{i,k};\xi_{i,k})-\eta\nabla f_i(\vx_{i,k})\|^2|\cF_{k-1}\Big]\leq 2\eta \|\tilde J_N\bar\vr\|^2+2\eta\|\tilde J_N(\vr-\bar\vr)\|^2\\
&\leq 2\eta \|J_N\bar\vr\|^2+2\eta \|\vr-\bar\vr\|^2\leq 2\eta N\sigma^2+2\eta Lc_l\|X_k-\bone\bar\vx_k^\T\|^2
\ena
where we used \cref{A2} such that $\|\nabla l(\vx;d)-\nabla l(\vy;d)\|\leq c_l\|\vx-\vy\|$; We also used $\|\tilde J_N\|^2\leq\|\tilde J_N\|_1\|\tilde J_N\|_\infty=1$ and  
\bee
\|\vr-\bar\vr\|^2\leq c_l^2\sum_{i=1}^n\sum_{u=1}^{N_i}\|\vx_{i,k}-\bar\vx_k\|^2\leq c_l^2\max_i N_i\sum_{i=1}^n\|\vx_{i,k}-\bar\vx_k\|^2\leq c_lL\|X_k-\bone\bar\vx_k^\T\|^2.
\ene

\bibliographystyle{siamplain}
\bibliography{mybibf}
\end{document}